\documentclass{article}

\PassOptionsToPackage{numbers, compress}{natbib}

 \usepackage[preprint]{neurips_2025}

\usepackage[utf8]{inputenc} 
\usepackage[T1]{fontenc}    
\usepackage[hypertexnames=false]{hyperref}       
\usepackage{url}            
\usepackage{booktabs}       
\usepackage{amsfonts}       
\usepackage{nicefrac}       
\usepackage{microtype}      
\usepackage{xcolor}         
\usepackage{amsmath}
\usepackage{amsthm}
\usepackage{amssymb}
\usepackage{booktabs}
\usepackage{makecell}

\usepackage{subcaption}
\usepackage{wrapfig}
\usepackage{algorithm}
\usepackage[noend]{algorithmic}
\usepackage{tikz}
\usepackage{enumitem}
\usepackage{graphicx}
\usepackage[export]{adjustbox}
\usetikzlibrary{arrows,automata,positioning}
\usetikzlibrary{fit,shapes}

\newcommand{\ignore}[1]{}
\newcommand{\LLL}{{\cal L}}

\newcommand{\YYY}{{\cal Y}}
\newcommand{\any}{\Sigma}
\newcommand{\lwa}{\mathbb{W}}
\newcommand{\axp}{\mathsf{AXp}}
\newcommand{\cxp}{\mathsf{CXp}}
\newcommand{\ffa}{\mathsf{FFA}}
\newcommand{\nosize}[1]{}

\newcommand{\pjs}[1]{\textcolor{blue}{\textsc{Pjs:} #1}}

\newcommand{\noffa}[1]{}
\renewcommand{\noffa}[1]{#1}

\newtheorem{example}{Example}
\newtheorem{definition}{Definition}
\newtheorem{proposition}{Proposition}

\title{A Formal Framework for the Explanation of \\ Finite Automata Decisions}

\author{%
  Jaime Cuartas Granada \quad Alexey Ignatiev \quad Peter J. Stuckey \\
  Department of Data Science and AI, Faculty of IT\\
  Monash University, Melbourne, Victoria, Australia \\
  \texttt{\{jaime.cuartasgranada, alexey.ignatiev, peter.stuckey\}@monash.edu} \\
}

\begin{document}

\maketitle

\begin{abstract}
  Finite automata (FA) are a fundamental computational abstraction
    that is widely used in practice for various tasks in computer
    science, linguistics, biology, electrical engineering, and
    artificial intelligence.
    Given an input word, an FA maps the word to a result, in the
    simple case ``accept'' or ``reject'', but in general to one of a
    finite set of results.
    A question that then arises is: why?
    Another question is: how can we modify the input word so that it
    is no longer accepted?
    %
    %
    %
    One may think that the automaton itself is an adequate explanation
    of its behaviour, but automata can be very complex and difficult
    to make sense of directly.
    In this work, we investigate how to explain the behaviour of an FA
    on an input word in terms of the word's characters.
    In particular, we are interested in \emph{minimal} explanations:
    what is the minimal set of input characters that explains the
    result, and what are the minimal changes needed to alter the
    result?
    Note that multiple minimal explanations can arise in both cases.
    In this paper, we propose an efficient method to determine all
    minimal explanations for the behaviour of an FA on a particular
    word. This allows us to give unbiased explanations about which input
    features are responsible for the result.
    Experiments show that our approach scales well, even when the
    underlying problem is challenging.
\end{abstract}

\section{Introduction}

With the rise of Artificial Intelligence, we have seen increasing use
of black-box systems to make decisions.
This in turn has led to the demand for eXplainable Artificial
Intelligence (XAI)~\cite{miller-aij19,BARREDOARRIETA202082,xrlSurvey}
where the decisions of these systems must be explained to a person
either affected by or implementing the decisions.
While much of the work on XAI has a very fuzzy or adhoc definition of
an \emph{explanation} for a decision, Formal Explainable Artificial
Intelligence (FXAI)~\cite{Darwiche2023,delivering_trust}, demands
rigorous explanations that satisfy formal properties, including
usually some form of minimality.

The demand for explanations of behaviour is not only a question of
interest for black-box systems, but in fact any system that makes a
decision, including those that are deemed inherently
interpretable~\cite{msi-fai23}.
Motivated by the widespread use of finite automata (FA), this paper
examines how to \emph{formally} explain the result of their operation.
While FA are sometimes seen as one of the simplest models of
computation, their inner workings are often far from obvious.
Understanding their behavior is beneficial in domains where FA are
used widely as components of more complex tasks, e.g. string
searching, pattern matching, lexical analysis, or deep packet
inspection~\cite{Dpi_survey}.
Furthermore, agent policies in reinforcement learning, once trained,
are often represented as finite automata; the need to understand why a
particular action was taken by the agent in a given situation and
whether or not a different outcome could be reached by taking another
action further underscores the need for FA explanations.

In light of the above, given an automaton $\mathcal{A}$ and an input string $w$,
where $\mathcal{A}$ applied to $w$ returns  $r \in
\{\text{accept},\text{reject}\}$;
this paper builds on the apparatus of~\cite{delivering_trust} to
answer the following questions \emph{formally}:
\begin{itemize}
  \item \emph{Abductive explanation}: why does $\mathcal{A}$ return $r$
    for $w$, i.e. what features of $w$ are necessary to guarantee the
    result~$r$.
  \item  \emph{Contrastive explanation}: how can I modify $w$
    minimally obtaining $w'$ s.t.\ $\mathcal{A}$ returns a different result
    $r' \neq r$ on $w'$.
\end{itemize}

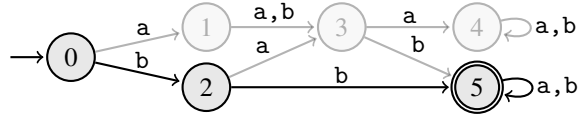
\begin{figure}[t]
  \centering
    \begin{tikzpicture}[shorten >=1pt, node distance=1cm, on grid, auto, initial text=, thick]
  \tikzstyle{every state}=[fill={rgb:black,1;white,10}, minimum size=18pt, inner sep=1pt]

  \node[state, initial]            (q0) {0};
  \node[state, opacity=0.3]        (q1) [right=1.8cm of q0, yshift=0.4cm] {1};
  \node[state]                     (q2) [right=1.8cm of q0, yshift=-0.4cm] {2};
  \node[state, opacity=0.3]        (q3) [right=1.8cm of q1] {3};
  \node[state, opacity=0.3]        (q4) [right=1.8cm of q3] {4};
  \node[state, accepting, double]  (q5) [right=3.6cm of q2] {5};

  \path[->]
    (q0) edge[draw opacity=0.3] node[xshift=8pt, yshift=-1pt] {\texttt{a}} (q1)
         edge node[xshift=-5pt, yshift=-2pt] {\texttt{b}} (q2)
    (q1) edge[draw opacity=0.3] node[yshift=-2pt] {\texttt{a,b}} (q3)
    (q2) edge node[yshift=-2pt] {\texttt{b}} (q5)
         edge[draw opacity=0.3] node[xshift=2pt, yshift=-2pt] {\texttt{a}} (q3)
    (q3) edge[draw opacity=0.3] node[yshift=-2pt] {\texttt{a}} (q4)
         edge[draw opacity=0.3] node[xshift=-3pt, yshift=-2pt] {\texttt{b}} (q5)
    (q4) edge[loop right, draw opacity=0.3] node[xshift=-2pt] {\texttt{a,b}} (q4)
    (q5) edge[loop right] node[xshift=-2pt] {\texttt{a,b}} (q5);
\end{tikzpicture}
    \caption{
      Input \texttt{bbbbb} is accepted by this automata.
      Opaque states indicate transitions that are not traversed
      for the input \texttt{bbbbb}.
      Two valid explanations are \texttt{\underline{bb}bbb} and
      (a shorter one) \texttt{bb\underline{b}bb}.}
  \label{fig:bbbbb_trace}
\end{figure}

Interestingly, while the latter question can be related to the
well-known problem of (and algorithms for) computing the \emph{minimum
edit distance}~\cite{wagner-cacm74}, no similar mechanisms exist for
answering the former `why' question.
Moreover, as formal abductive and contrastive explanations enjoy the
minimal hitting set duality relationship~\cite{delivering_trust}, it
is often instrumental in devising efficient algorithms for their
computation and enumeration, which further underscores the
advantage of formal explanations studied in this paper over the
minimal edit distance.
Also, note that many approaches to XAI actually consider finite
automata as explanations themselves \cite{re_understanding_FA,
deepsynth_fa}, using them as (approximate) explanations for black box
models.
But while FA are a simple computational mechanism, their behaviour
when viewed from outside may not be obvious, and indeed in our view of
explanations, we are trying to find features of the \emph{input} which
minimally explain the result of the automaton.

\begin{example}\label{ex:bbbbb}
  Consider the deterministic FA shown in Figure~\ref{fig:bbbbb_trace}.
  Observe that it accepts the input word \texttt{bbbbb}.
  How should we explain this?
  One obvious way is to trace the execution of the automaton on the
  input word.
  Clearly, reading the first two \texttt{b}'s leads to an accepting
  state, so an explanation could be \texttt{\underline{bb}bbb} where
  the underlined characters are those that explain the acceptance.
  But other (shorter) explanations exist, e.g.
  \texttt{bb\underline{b}bb} is a correct explanation, as any word of
  length 5 with a \texttt{b} in position 3 will be accepted.
  \qed
\end{example}

\section{Preliminaries} \label{sec:prelim}

\newcommand{\lfa}{\Sigma}

\subsection{Finite Automata and Regular Expressions}

Here, we adopt standard definitions and notations for \emph{finite
automata}~\cite{FA_and_their_decision_problems, FA_languages,
PERRIN19901}.
Let $\Sigma$ be a finite alphabet of symbols.
A \textit{word} may be the \textit{empty word} $\epsilon$ or a finite
sequence of symbols over the alphabet $\Sigma$ written as $w=w_1w_2\ldots
w_n$ s.t.\ $w_i\in \Sigma$ for all $1\leq i \leq n$.
Integer $n$ here is referred to as the \emph{length} of the word $w$,
i.e. $\lvert w\rvert=n$.
We use array notation to lookup symbol in a word, i.e. $w[i]$ is the
$i^{th}$ symbol in the string, with $1\leq i \leq n$ .
%
%
Given a finite alphabet $\Sigma$, the set of \emph{all} words over
$\Sigma$ is denoted as $\Sigma^*$.
A \emph{langugage} $L$ is a subset of $\Sigma^*$.
For any $k\in \mathbb{N}$, $\Sigma^k$ is defined as: $\Sigma^k=\{x\mid
x \in \Sigma^* \text{ and } \lvert x\rvert=k\}$, where
$\Sigma^0=\{\epsilon\}$.

A finite automaton (FA) is a tuple $\mathcal{A} = ({Q, \Sigma, \delta, q_0, F})$
where $\Sigma$ is a finite alphabet;
$Q$ is a finite non-empty set of \emph{states} including the
\emph{initial state} \(q_0\) and a set \(F\) of \emph{accepting
states};\footnote{While the approach we consider can be applied to
    finite classifiers which return a result $r \in R$ from a set, the
    explanations we consider either aim to guarantee the same result
    $r$ or lead to \emph{any} different result $r' \in R \setminus
    \{r\}$, thus treating all results $R \setminus \{r\}$
equivalently, so only two output results are ever required. Hence we
restrict to classic automata.} and $\delta \subseteq Q \times \Sigma
\times Q$ is a set of \emph{transitions}.
If there is $(q,c,q') \in \delta$ such that $q,q'\in Q$ and $c\in
\Sigma$, then we say that there is a \emph{transition} from state $q$
to state $q'$ on symbol $c$ written as $q \xrightarrow{c} q'$.
A \emph{computation} for string $w$ of length $n\triangleq\lvert
w\rvert$ in an automaton $\mathcal{A}$ is a sequence of transitions $q_0
\xrightarrow{w[1]} q_1 \xrightarrow{w[2]} q_2 \xrightarrow{w[3]}
\cdots \xrightarrow{w[n]} q_n$ where $(q_{i-1},w[i],q_{i}) \in \delta$
and $1\leq i \leq n$.
An \emph{accepting computation} for $w$ in $\mathcal{A}$ from state $q_0$ is a
computation for $w$ in the automaton $\mathcal{A}$ where $q_n \in F$.
The \emph{regular language} of automaton $\mathcal{A}$ (often said to be
\emph{recognised} by $\mathcal{A}$) from state $q \in Q$, $L(q,\mathcal{A})$, is the set
of strings $w$, which have an accepting computation from state~$q$.
The \emph{language} of automaton $\mathcal{A}$ is defined as $L(\mathcal{A}) = L(q_0,\mathcal{A})$.
We define the size of automaton $\mathcal{A}$ denoted $|\mathcal{A}|$ as its number $|Q|$ of states.

A finite automaton is \textit{deterministic} (DFA) if for each pair
$(q, c)\in Q\times\Sigma$ there is at most one state $q' \in Q$ such
that $(q, c, q')\in \delta$.
Finite automata not satisfying the above condition are called
\textit{non-deterministic} (NFA).
For any NFA, there exists a DFA recognising the same regular
language~\cite{PERRIN19901}.
Due to this fact, this work focuses on DFAs.

Given a DFA $\mathcal{A}$ over the alphabet $\Sigma$, its \textit{complement} is
another FA, denoted as $\overline{\mathcal{A}}$, over the same alphabet
$\Sigma$, and it recognises the language
$L(\overline{\mathcal{A}}) = \overline{L(\mathcal{A})}$, and
$\overline{L(\mathcal{A})} = \Sigma^* \setminus L(\mathcal{A})$.
It is well understood how to construct a DFA for $\overline{\mathcal{A}}$
from $\mathcal{A}$ with at most one more state.
Clearly,  for any $\mathcal{A}$, $L(\mathcal{A})\cap
L(\overline{\mathcal{A}}) = \emptyset$, and
$L(\mathcal{A})\cup L(\overline{\mathcal{A}})=\Sigma^*$.
As a consequence, for any word
$w\in \Sigma^*$, $w\in L(\mathcal{A}) \iff w\not\in L(\overline{\mathcal{A}})$.

We make use of standard regular expression notation.
We denote by $\emptyset$ the empty language, by $c \in \any$ the
regular expression defining the language $\{c\}$, and by $\any$ the
language $\{\Sigma\}$, viewed as the set of all strings of length 1.
In general, given a regular expression $R$, the language it defines is
denoted by $L(R)$.
The concatenation of two regular expressions $R_1 R_2$ denotes the
regular language $\{ w_1 w_2 \mid w_1 \in L(R_1), w_2 \in L(R_2)\}$.
The union of two regular expressions $R_1 | R_2$ denotes the language $L(R_1) \cup
L(R_2)$.
The Kleene star $R^*$ expression is the language defined as $\{ w_1
\cdots w_n \mid n \geq 0, w_i \in L(R) \}$.
Given a \emph{range} $l..u$ to denote the set of non-negative integers $\{i \in
\mathbb{N} \cup \{0\} \mid l \leq i \leq u\}$,
we define extended regular expression notation
$\Sigma_l^u$ defining all strings of length at least $l$ and at most $u$,
defined as $\Sigma_l^u \equiv \overbrace{\Sigma \cdots \Sigma}^{l} \overbrace{(\Sigma|\emptyset) \cdots
(\Sigma|\emptyset)}^{u-l}$, $u < \infty$
and
$\Sigma_l^u \equiv \overbrace{\Sigma \cdots \Sigma}^{l} \Sigma^*$, $u = \infty$.
Observe that $\Sigma^*= \Sigma_0^\infty$ and
$\Sigma^+ = \Sigma_1^\infty = \Sigma^* \setminus \Sigma^0$.

We denote by $\mathcal{A}$ a Deterministic Finite Automaton (DFA)
and by $R$ a regular expression.
Their recognised languages, or sets of words, are written as
$L(\mathcal{A})$ and $L(R)$, respectively.

\subsection{Classification Problems and Formal Explanations}

Following~\cite{delivering_trust,darwiche2023logic}, we assume
classification problems to be defined on a set $\mathcal{F}$ of $m$
features and a set $\mathcal{K}$ of $k$ classes.
Each feature $i\in\mathcal{F}$ is in some domain $\mathbb{D}_i$ while
the feature space is $\mathbb{F}=\prod_{i=1}^{m} \mathbb{D}_i$. A
classifier is assumed to compute a total function $\kappa:
\mathbb{F}\rightarrow \mathcal{K}$.

Next, we assume an instance $\mathbf{v}\in\mathbb{F}$ such that
$\kappa(\mathbf{v})=~c\in~\mathcal{K}$.
We are interested in explaining why $\mathbf{v}$ is classified as
class $c$.
Given an instance $\mathbf{v}\in\mathbb{F}$ such that
$\kappa(\mathbf{v})=c\in\mathcal{K}$, an \emph{abductive explanation}
(AXp) $\mathcal{X}\subseteq \mathcal{F}$ is a minimal subset of
features \emph{sufficient} for the prediction.
Formally, $\mathcal{X}$ is defined as:
\begin{equation}
    \label{eq:axp}
    \forall (\textbf{x}\in \mathbb{F}). \left[\bigwedge\nolimits_{i
    \in \mathcal{X}} (x_i=v_i)\right]\rightarrow
    \left(\kappa(\mathbf{x})=c\right)
\end{equation}

Similarly, one can define another kind of explanation.
Namely, given an instance $\mathbf{v}\in\mathbb{F}$ such that
$\kappa(\mathbf{v})=c$, a \emph{contrastive explanation} (CXp) is a
minimal subset of features $\mathcal{Y}\subseteq\mathcal{F}$
that, if allowed to change, enables the prediction's alteration.
Formally, a contrastive explanation $\mathcal{Y}$ is defined as follows:
\begin{equation}
    \label{eq:cxp}
    \exists (\textbf{x}\in \mathbb{F}). \left[\bigwedge\nolimits_{j
    \not\in \mathcal{Y}} (x_j=v_j)\right] \land
    \left(\kappa(\mathbf{x}) \neq c\right)
\end{equation}
Observe that abductive explanations are used to explain \emph{why} a
prediction is made by the classifier $\kappa$ for a given instance while
contrastive explanations can be seen to answer \emph{why not} another
prediction is made by $\kappa$.
Alternatively, CXps can be seen as answering \emph{how} the
predication can be changed.

Importantly, abductive and contrastive explanations are known to enjoy
a minimal hitting set duality relationship~\cite{axpcxp}.
Given $\kappa(\mathbf{v})=~c$, let $\mathbb{A}_\mathbf{v}$ be the
complete set of AXps and $\mathbb{C}_\mathbf{v}$ be the complete set
of CXps for this prediction.
Then each AXp $\mathcal{X}\in\mathbb{A}_\mathbf{v}$ is a minimal
hitting set of $\mathbb{C}_\mathbf{v}$ and, vice versa, each CXp
$\mathcal{Y}\in\mathbb{C}_\mathbf{v}$ is a minimal hitting set of
$\mathbb{A}_\mathbf{v}$.\footnote{Given a collection of sets
$\mathbb{S}$, a \emph{hitting set} of $\mathbb{S}$ is a set $H$ such
that for each $S \in \mathbb{S}$, $H \cap S \neq \emptyset$. A
hitting set is \emph{minimal} if no proper subset of it is a
hitting set.}
This fact is the basis for the algorithms used for formal explanation
\emph{enumeration}~\cite{delivering_trust,ffa}.

\noffa{
A challenge arising in formal XAI
is given there are many AXps or CXps,
which is the \emph{best} explanation to present to a user.
An obvious answer is to use one of minimal size,
but an alternate answer is to
generate a result that records the
influence of all explanations.
A \emph{formal feature attribution} (FFA)
weighs the importance of each
feature \cite{ffa,yu-sat24} in determining the result that $\kappa(\mathbf{v}) = c$.
We define the formal feature attribution of feature $i$ as
\begin{equation}
    \label{eq:ffa_definition}
    \ffa_{\mathbf{v}}(i) = \frac{ |\{ {\cal X} ~|~ {\cal X} \in \mathbb{A}_{\mathbf{v}}, i \in {\cal X} \}|
}{|\mathbb{A}_{\mathbf{v}}|}
\end{equation}
The proportion of all AXps shows the (non-zero)
contribution of each feature to the decision.
A \emph{formal feature attribution} for $\kappa(\mathbf{v}) = c$ is the set
$\{ (i, \ffa_{\mathbf{v}}(i)) ~|~ i \in
{\cal F}, \ffa_{\mathbf{v}}(i)) > 0 \}$.
}
Critically, formal feature attribution gives an \emph{unbiased} measure of
the influence of input features on the decision.

\section{Explainable Finite Automata}\label{sec:explain}

In this work, we aim to explain the behaviour of a finite automaton
$\mathcal{A}$ on an input $w\in\Sigma^*$, which can be viewed as a
classifier mapping input $w$ to a class in ${\cal K} = \{
\mbox{accept}, \mbox{reject} \}$.
Similar to classification problems, we propose two types of
explanations for FA: abductive explanations (AXps) and contrastive
explanations (CXps).
Informally, an AXp answers ``why does $\mathcal{A}$ accept/reject
$w$?''
a CXp answers ``How can $w$ be modified to alter the response
of~$\mathcal{A}$?''.

We first define explanations abstractly, then consider concrete
instantiations.
An explanation is a regular language $L$ from a set ${\cal L}(w,\mathcal{A})$ of
candidate explanations, defined over the alphabet $\Sigma$ and
depending on both $w$ and $\mathcal{A}$. We refer to ${\cal L}(w,\mathcal{A})$ as the
\emph{language of explanations}.\footnote{We restrict ourselves to
    explanations for $w \in L(\mathcal{A})$. Observe that we can explain $w
\not\in L(\mathcal{A})$ by explaining $w \in L(\overline{\mathcal{A}})$.}

\begin{definition}[Abductive Explanation - AXp]
    A \emph{weak AXp} for $w\in~L(\mathcal{A})$ is a language $\mathcal{X}
    \in \mathcal{L}(w, \mathcal{A})$ such that $w \in \mathcal{X}$ and
    $\mathcal{X} \subseteq L(\mathcal{A})$.
    A (minimal) \emph{AXp} for $w\in L(\mathcal{A})$ is a language $\mathcal{X} \in {\cal
    L}(w,\mathcal{A})$ where there is no weak AXp $\mathcal{X}' \in
    {\cal L}(w,\mathcal{A})$ of $w \in L(\mathcal{A})$ such that
    $\mathcal{X}' \supset \mathcal{X}$.
\end{definition}
AXps capture maximality under set inclusion over
languages in $\mathcal{L}(w, \mathcal{A})$.
This property allows an AXp to define the
\emph{most general} explanation for why $w \in L(\mathcal{A})$.

\begin{definition}[Contrastive Explanation - CXp]
    A \emph{weak CXp} for $w\in~L(\mathcal{A})$ is a language $\mathcal{Y}
    \in \mathcal{L}(w, \mathcal{A})$ such that $w \in \mathcal{Y}$ and
    $\mathcal{Y} \not\subseteq L(\mathcal{A})$.
    A (minimal) \emph{CXp} for $w\in L(\mathcal{A})$ is a language $\mathcal{Y} \in {\cal
    L}(w,\mathcal{A})$ where there is no weak CXp $\mathcal{Y}' \in
    {\cal L}(w,\mathcal{A})$ of $w \in L(\mathcal{A})$ such that
    $\mathcal{Y}' \subset \mathcal{Y}$.
\end{definition}
In other words, a CXp is a \emph{minimal language} including $w$
and a word \emph{not accepted} by $\mathcal{A}$.

These abstract definitions hinge upon the definitions of the language
of explanations $\LLL(w,\mathcal{A})$.
Clearly, if $\LLL(w,\mathcal{A})$ includes all regular languages then $w$ itself
defines its own AXp $L(w)$, and any $w|w'$ with $w' \not\in L(\mathcal{A})$
defines a CXp $L(w|w')$.
Neither of these offers a \emph{meaningful} explanation.

Let us define a class of explanation languages $\lwa_l^u \triangleq
\LLL(w,\mathcal{A})$ where $0 \leq l \leq 1 \leq u$, defined by regular
expressions formed by replacing some characters in $w$ by
$\any_{l}^{u}$, i.e.\ replacing a character in $w$ by \emph{any}
string of characters of length between $l$ and $u$ (inclusive).
Note that since $l \leq 1 \leq u$, for all $\mathcal{Z} \in \lwa_l^u$
it holds that $w \in \mathcal{Z}$. We focus on the case where $l = u =
1$, i.e. when a character in $w$ may be replaced by \emph{any} single
character.

The motivation for considering $\lwa_1^1$ is that the \emph{retained}
characters in $w$, i.e.\ those not replaced by $\any$, give us an
understanding of what is essential in $w$ for its acceptance, while
the \emph{modified} parts of the word are \emph{irrelevant} to the
automaton's decision.
Note that in $\lwa_1^1$, all strings are of length equal to $\lvert
w\rvert$,
and their deterministic automaton can be built trivially
with $\lvert w\rvert +1$ states.

\begin{figure}[t]
  \centering
    \begin{tikzpicture}[shorten >=1pt, node distance=1cm, on grid, auto, initial text=, thick]
    \tikzstyle{every state}=[fill={rgb:black,1;white,10}, minimum size=18pt, inner sep=1pt]
    \node[state, initial]           (q0)                                  {0};
    \node[state]                    (q1)  [below=of q0]                   {1};
    \node[state]                    (q2)  [right=1.4cm of q0, yshift=1cm] {2};
    \node[state, accepting, double] (q3)  [right=1.4cm of q1]             {3};
    \node[state]                    (q4)  [right=1.4cm of q2]             {4};
    \node[state]                    (q5)  [below=of q4]                   {5};
    \node[state]                    (q6)  [right=1.4cm of q3]             {6};
    \node[state]                    (q7)  [right=1.4cm of q4]             {7};
    \node[state, accepting, double] (q8)  [right=1.4cm of q5]             {8};
    \node[state, accepting, double] (q9)  [right=1.4cm of q6]             {9};
    \node[state, accepting, double] (q10) [right=1.4cm of q7]             {10};

    \path[->]
        (q0) edge             node {\texttt{a}}     (q2)
             edge             node {\texttt{b}}     (q1)
        (q1) edge             node {\texttt{c}}     (q3)
        (q2) edge             node {\texttt{b}}     (q4)
        (q3) edge[loop above] node {\texttt{c}}     (q3)
             edge             node {\texttt{d}}     (q6)
        (q4) edge             node {\texttt{c}}     (q7)
             edge[swap]       node {\texttt{[d-z]}} (q5)
        (q5) edge             node {\texttt{e}}     (q8)
        (q6) edge             node {\texttt{a}}     (q9)
        (q7) edge             node {\texttt{d}}     (q10)
             edge             node {\texttt{e}}     (q8)
        (q8) edge[loop right] node {\texttt{e}}     (q8)
        (q10) edge[loop right] node {\texttt{d}}    (q10);

\end{tikzpicture}
    \caption{
        DFA recognising the language~\texttt{(abcd+)|(ab[c-z]e+)|(bc+da)|(bc+)}.
    }
  \label{fig:dpi_axps}
\end{figure}

\begin{example}[AXp and CXp] \label{example_dot}
    Consider the automaton $\mathcal{A}$ in Figure~\ref{fig:dpi_axps} accepting
    the language \texttt{(abcd+)|(ab[c-z]e+)|(bc+da)|(bc+)} over~the
    al\-phabet $\Sigma=\{\texttt{a},\texttt{b},\dots,
    \texttt{z}\}$,~\footnote{Similar \emph{union} languages are
    often an object of study in Deep Packet
    Inspection~\cite{Dpi_survey}.} and the word $w=\texttt{accc}$,
    $w\not\in L(\mathcal{A})$.
    Using $\lwa_1^1$, a trivial \emph{weak} AXp is $L(\texttt{accc})$,
    retaining all symbols.
    More informative explanations are minimal AXps, such as
    $L(\texttt{ac}\any\any)$ and $L(\texttt{a}\any\any\texttt{c})$,
    each revealing a different reason why $w\not\in L(\mathcal{A})$.
    These AXps represent subset-minimal sets of symbols of $w$ that
    are sufficient to explain the rejection of \textit{w}.
    A minimal CXp is $L(\any \texttt{ccc})
    \not\subset~L(\overline{\mathcal{A}})$, e.g. a possible correction of
    rejection is $w'=\texttt{\underline{b}ccc}\in L(\mathcal{A})$.
    Another minimal CXp is $L(a\any\texttt{c} \any)$ since
    $w''=\texttt{a\underline{b}c\underline{d}}\in L(\mathcal{A})$. \qed
\end{example}

Similarly to FXAI~\cite{delivering_trust}, we define a shorthand
notation for explanations in $\lwa_l^u$ using subsets of positions in
$w \in L(\mathcal{A})$:
\begin{align*}
    \axp_l^u(S, w)= L(s_1 \cdots s_n) \mid s_i = w[i] \text{ if } i\in
    S \text{ else } \any_{l}^{u} \\
    \cxp_l^u(S, w)= L(s_1 \cdots s_n) \mid s_i = w[i] \text{ if }
    i\not\in S \text{ else } \any_{l}^{u}
\end{align*}
Note that $\any_{1}^{1}$, $\any_{1}^{\infty}$, and $\any_{0}^{\infty}$
correspond to $\any$, $\any^+$, and $\any^*$,
allowing replacements at position $i$
by any symbol, any non-empty string,
or any string (including empty), respectively.
%

\begin{example} \label{ex:indices}
    Consider the setup of Example~\ref{example_dot}.
    Given the word $w=\texttt{accc}$,
    $L(\texttt{ac}\any\any)=\axp_1^1(\{1,2\}, w)$ and
    $L(\texttt{a}\any\any\texttt{c})=\axp_1^1(\{1,4\}, w)$.
    Similarly, $L(\any\texttt{ccc})=\cxp_1^1(\{1\}, w)$ and
    $L(\texttt{a}\any\texttt{c}\any)=\cxp_1^1(\{2,4\}, w)$.
    \qed
\end{example}

This notation shows that AXps are minimal subsets of
retained character positions, and CXps are minimal subsets of freed
character positions.
This makes it easy to see the hitting set duality between
AXps and CXps for FA.

\begin{proposition}
    \label{thm:duality}
    (Proofs for this and all other propositions are provided in the Appendix)
    Given $w\in L(\mathcal{A})$,
    assume that the sets of all AXps and CXps are
    denoted as $\mathbb{A}=\{X\subseteq\{1,\ldots,|w|\} \mid 
    \axp_l^u(X, w)\in\mathbb{W}\}$ and
    $\mathbb{C}=\{Y\subseteq\{1,\ldots,|w|\} \mid \cxp_l^u(Y,
    w)\in\mathbb{W}\}$, respectively.
    It holds that each $X\in\mathbb{A}$ is a
    minimal hitting set of $\mathbb{C}$ and, vice versa, each
    $Y\in\mathbb{C}$ is minimal hitting set of
    $\mathbb{A}$.
\end{proposition}

Various explanation languages can be related as follows:

\begin{proposition}\label{rel-axp}
    Let $w\in L(\mathcal{A})$ and $l_1..u_1 \supseteq l_2..u_2$.
    If $S$ is such that
    $\axp_{l_1}^{u_1}(S,w)$ is a weak AXp for $w \in L(\mathcal{A})$ using
    $\lwa_{l_1}^{u_1}$
    then $\axp_{l_2}^{u_2}(S, w)$ is a weak AXp using $\lwa_{l_2}^{u_2}$.
\end{proposition}

\begin{example} \label{ex:allaxps}
    Consider language $L(R)$ defined by the regular expression
    $R=\texttt{(abcd+)|}$ $\texttt{(ab[c-z]e+)|(bc+da)|(bc+)}$ 
    depicted in in Figure~\ref{fig:dpi_axps} and
    $w=\texttt{ceccd}$.
    Observe that \mbox{$w\not\in L(R)$}, i.e. $w\in \overline{L(R)}$.
    Using $\lwa_1^1$, an AXp is $\axp_1^1(\{2\}, w) = L(\any
    \texttt{e} \any \any \any)$ since $L(\any \texttt{e} \any \any
    \any) \subseteq \overline{L(R)}$.
    However, a more general language $L(\any^+ \texttt{e} \any^+
    \any^+ \any^+)$ (using $\lwa_1^\infty$) is not an AXp as it
    contains a word $\texttt{abeeee} \not\in \overline{L(R)}$.
    $\axp_1^\infty(\{2,3\}, w)$ is an AXp, i.e., $L(\any^+ \texttt{ec}
    \any^+ \any^+)\subseteq \overline{L(R)}$ as it reveals
    substring \texttt{ec} to be a reason for rejection.
    \qed
\end{example}

Although a richer explanation language may
lead to larger AXps in terms of character positions,
these may sometimes offer better expressivity to a user
if that is of their concern.

\begin{proposition}\label{rel-cxp}
    Let $w\in L(\mathcal{A})$, and $l_1..u_1\subseteq l_2..u_2$.
    If $S$ is such that
    $\cxp_{l_1}^{u_1}(S,w)$ is a weak CXp for $w \in L(\mathcal{A})$ using
    $\lwa_{l_1}^{u_1}$
    then $\cxp_{l_2}^{u_2}(S, w)$ is a weak CXp using $\lwa_{l_2}^{u_2}$.
\end{proposition}

As a corollary of Propositions~\ref{rel-axp}--\ref{rel-cxp},
for $l_1..u_1\supseteq~l_2..u_2$
and each CXp $S$ using $\lwa_{l_1}^{u_1}$ there is a subset of
$S$ which is a CXp using $\lwa_{l_2}^{u_2}$, similarly for
each AXp $S$ using $\lwa_{l_1}^{u_1}$ there is a superset of
$S$ which is an AXp using $\lwa_{l_2}^{u_2}$.

\section{Why Formal Explanations?} \label{sec:why}

Since DFAs are often considered as one of the simplest models of
computation, one might wonder why we need formal explanations for
their behavior in the first place.
Indeed, it may be assumed that DFAs are inherently interpretable as
one can easily trace the path through the states of the DFA given a
particular word and claim the path to be the reason for acceptance or
rejection (see Example~\ref{ex:bbbbb}).
Here, we argue that this is not always the case because there are
languages whose (smallest size) DFAs are exponentially larger than the
regular expressions generating these languages.
Importantly, some of these languages admit short formal explanations.

\begin{proposition} \label{prop:edfa}
  Given an alphabet $\Sigma$ and a parameter $n\in\mathbb{N}$, there
  exists a family of regular languages $L_n$ over $\Sigma$ such that
  any DFA $\mathcal{A}_n$ for $L_n$ has least $2^n$ states while $\forall
  w\in\Sigma^*$ for any $\axp_1^1(S, w)$ and any $\cxp_1^1(S, w)$ it
  holds that $|S|\leq k$, for some constant $k\in\mathbb{N}$, $k<n$.
\end{proposition}
One might assume that a \emph{computation} 
(the sequence of states visited) can reveal
the explanation for the acceptance/rejection of a
word. However, a \emph{computation} provides a causal chain 
without isolating the necessary symbols for an
abductive explanation.
To illustrate the deficiency of \emph{computation}
inspection, we revisit the automaton $\mathcal{A}$
from Example~\ref{ex:bbbbb}.
Figure~\ref{fig:bbbbb_trace} shows the DFA for
$\mathcal{A}$ while processing the word $w = bbbbb$.
The computation
requires the user to follow the path 
$0 \xrightarrow{b} 2 \xrightarrow{b} 5 \xrightarrow{b} 5 \xrightarrow{b} 5 \xrightarrow{b} 5$,
one might incorrectly identify the first two symbols ``b'' as the reason
for reaching the acceptance state $5$.
In contrast, our approach also identifies a more succinct 
explanation with smaller cardinality:
the symbol at position 3, which is an AXp, 
$\axp_1^1(\{3\}, w)=L(\any\any\texttt{b}\any\any)$. 

A plausible alternative to formal explanations for DFAs builds on
another well-studied problem in the context of languages, namely, the
computation of a \emph{minimum edit distance}~\cite{wagner-cacm74}.
Given a language $L$ and a word $w\not\in L$, the minimum edit
distance from $w$ to $L$ is the smallest number of symbols to replace
in $w$ in order to obtain a word $w'$ such that $w'\in L$.
While minimum edit distance can seemingly be related with smallest
size CXps, it is often not informative enough to explain the behavior
of a DFA, thus warranting the needs for AXps:

\begin{proposition} \label{prop:ed}
  Given an alphabet $\Sigma$ and a parameter $n\in\mathbb{N}$, there
  exists a family of regular languages $L_n$ over $\Sigma$ and a word
  $w\not\in L_n$ with the minimum edit distance of at least $n$ such
  that for any AXp $\axp_1^1(S,w)$ it holds that $|S|\leq k$, for some
  constant $k\in\mathbb{N}$, $k<n$.
\end{proposition}

As a result, understanding the behavior of a DFA either by directly
inspecting it or by relying on the minimum edit distance may be quite
misleading in practice.

The apparatus of formal explainability discussed hereinafter offers a
potent alternative enabling a user to opt for succinct abductive or
contrastive explanations, or (importantly) both.
Crucially and following~\cite{delivering_trust}, besides minimality,
formal abductive (resp., contrastive) explanations for DFAs guarantee
logical correctness as they can be seen to satisfy the property
\eqref{eq:axp} (resp., \eqref{eq:cxp}).
The following sections describe the algorithms for computing formal
explanations and demonstrate their practical scalability.

\section{Computing One Formal Explanation}\label{sec:axp}




While computing an AXp is generally hard in the context of
classification problems~\cite{delivering_trust}, this section shows it
to be straightforward for DFAs.

Given a word $w\in L(\mathcal{A})$, Algorithm~\ref{alg:extract_axp} extracts a
minimal AXp by refining a \emph{weak} AXp, initially $L(w)$.
The algorithm ``drops'' characters from $w$ replacing them one by one
with $\any_l^u$, iterating only over indices in the candidate set $X$.
If the resulting language is still included in $L(\mathcal{A})$, the character
is considered irrelevant and permanently ``dropped''; otherwise, it is
part of the resulting AXp.

The output is minimal, as ``dropping'' any remaining character would
necessarily introduce a language with at least one counterexample.

\begin{algorithm}[tb]
  \caption{\textsc{ExtractAXp} -- a Single AXp Extraction}
\label{alg:extract_axp}
\textbf{Input}: Candidate set $X$, automaton $\mathcal{A}$, word $w$, bounds $l, u$ \\
\textbf{Output}: Minimal $X$ with $\axp_l^u(X, w) \subseteq L(\mathcal{A})$

\begin{algorithmic}[1]
  \FORALL{$i \in X$} \label{alg:line:loop}
    \IF{$\axp_l^u(X \setminus \{i\}, w) \subseteq L(\mathcal{A})$}
      \STATE $X \gets X \setminus \{i\}$
    \ENDIF
  \ENDFOR
  \STATE \textbf{return} $X$
\end{algorithmic}
\end{algorithm}


\begin{proposition}\label{thm:axp}
  Given a DFA $\mathcal{A}$ with $m$ states and a word $w$ s.t. $\lvert
  w\rvert=n$, computing an AXp for $w \in L(\mathcal{A})$ in languages
  $\lwa_1^1$, $\lwa_{1}^\infty$, or $\lwa_{0}^\infty$ can be done in
  $O(\lvert\Sigma\rvert mn^2)$ time.
\end{proposition}

But if the language is given as a regular expression instead of a DFA, 
extracting just one AXp may be \emph{intractable}.

\begin{proposition}
  Computing a single AXp using $\lwa_0^\infty$ for $w \in L(R)$ given
  a regular expression $R$ is PSPACE-hard.
\end{proposition}

Note for DFAs, CXp extraction for $\lwa_1^1$, $\lwa_1^\infty$, and $\lwa_0^\infty$
can be done similarly in $O(|\Sigma|mn^2)$ time.
\ignore{
Similarly to \textsc{ExtractAXp}, procedure \textsc{ExtractCXp} has
the complexity of
\begin{corollary}
    Given a DFA $A$ with $m$ states and a word $w$ s.t. $\lvert
    w\rvert=~n$, finding a CXp for $w \in L(A)$ can be done in
    $O(\lvert\Sigma\rvert mn^2)$ time.
\end{corollary}
}

\ignore{
\section{Propositional Encoding} \label{sec:sat_encoding}

Here we describe a propositional encoding of the operation of a DFA
$R$ on words of length $k$.
This encoding can be used to compute AXps and CXps for accepted or
rejected words by a DFA, with the use of modern SAT
solving~\cite{biere2021handbook}.

Given a DFA $R=({Q,\Sigma, \delta, q_0, F})$ and a length $k$, the set
of variables to encode all accepted computations of size $k$, $q_0
\xrightarrow{w[1]} q_1 \xrightarrow{w[2]} q_2 \xrightarrow{w[3]} \dots
\xrightarrow{w[k]} q_k$ is shown below followed by
Algorithm~\ref{tab:encoding_computation} describing the constraints used:
\begin{enumerate}
    \item $x_{i,c}$, with $1\leq i \leq k$ and $c\in \Sigma$.
        $x_{i,c}=1$ iff char $w[i]=c$. They represent the symbols of
        the computation.
    \item $y_{i,j}$, with $1\leq i \leq k+1$ and $j\in Q$. $y_{i,j}=1$
        iff state $q_i=j$. They represent the states of the
        computation.
\end{enumerate}

\begin{table}[ht]
\begin{center}
\begin{tabular}{ c c c l }
 \hline
 & Constraint  & Range & \multicolumn{1}{c}{Semantic} \\
 \hline
 (1) & $y_{1,q_1}$ & & The first state is $q_1$ for position $0$ \\
 \hline
 (2) & $(y_{i,j}\land x_{i,c}) \rightarrow y_{i+1,j'}$ & \makecell{$i\in\{1,\ldots,k\}$,\\$j\in Q, c\in \Sigma$,\\$j'=\delta(j,c)$} & \makecell[l]{
 A state $q_{i-1}$ and symbol $w[i]$ in the \\ sequence of computations force the next \\ state to be $q_i=\delta(q_{i-1, w[i]})$}\\
 \hline
     (3) & $\sum\limits_{c\in \Sigma} x_{i,c} = 1$ & $i\in \{1,\ldots,k\}$ & \makecell[l]{
 There is exactly one symbol $c$ for each \\ position $i$ in a word}\\
 \hline
         (4) & $\sum\limits_{j\in Q} y_{i,j} = 1$ & $i\in \{1,\ldots,k\}$ & \makecell[l]{
 There is exactly one state $j$ for each position \\ $i$ in the sequence of the computation}\\
 \hline
 (5) & $\bigvee\limits_{f \in F} y_{k, f}$ & & \makecell[l]{The final state $q_k$ in the sequence of \\ computations is one of the accepting states}\\
 \hline
\end{tabular}
\end{center}
\caption{SAT encoding for a computation in a DFA}
\label{tab:encoding_computation}
\end{table}

In order to compute explanations for a given DFA $D=({Q,\Sigma,
\delta, q_0, F})$ and an accepted word $w$, let us consider the
constraints from Table~\ref{tab:encoding_computation} as \emph{hard}
constraints $\mathcal{H}$ for the complement of $D$, $R=({Q,\Sigma,
\delta, q_0, \overline{F}})$, with $k=|w|$.
Additionally, consider the set of \emph{soft} unit clauses
$\mathcal{S}=\{(x_{i, w[i]}) \mid i \in \{1,\ldots,k\}\}$ encoding the
input word $w$.
Note that if $\mathcal{S}$ holds, the input word is fixed, and the
first constraints (1-4) force $D$ and $R$ to follow the same
computation, as they share the same states, alphabet, transitions, and
initial state, ultimately reaching the same final state $q_k$.
Since $w\in L(D)$, we have $q_k\in F$, which implies that $q_k\not\in
\overline{F}$ and the last constraint (5) is unsatisfiable.

Thus, to explain $w\in L(D)$, we use the \textit{hard} constraints
$\mathcal{H}$ encoding the complement automaton $R=\overline{D}$, and
\textit{soft} constraints $\mathcal{S}$ encoding the input word. Since
$w\not\in L(R)$, we get the following unsatisfiable formula:
$$\mathcal{S} \land \mathcal{H}\models \bot$$
For this unsatisfiable formula, and given the definitions of AXp and
CXp, \cite{axpcxp} has proved that AXp are $MUS$ and CXp are $MCS$.

Using this approach, we can compute AXp and CXp for DFA through MUS
and MCS enumeration.
In the following sections, we describe a more efficient model to
compute equivalent explanations considering NFA, also exploiting the
duality relationship between AXp and CXp as we can do here.

\begin{example}
    Consider the automaton $A_1$ depicted in Figure~\ref{fig:fa_complement_copy}
    and the word $w=\texttt{aab}\in L(A_1)$.
    To explain why $w\in L(A_1)$, we encode the constraints from Table~\ref{tab:encoding_computation}
    as a SAT formula for its complement $A_2$, also depicted in Figure~\ref{fig:fa_complement_copy}.

    Firstly, we encode the word $w$ as a set of soft constraints $\mathcal{S}=\{(x_{1,a}), (x_{2,a}), (x_{3,b})\}$.
    Then, we encode the constraints from Table~\ref{tab:encoding_computation} as hard constraints $\mathcal{H}$.

    The figure Figure~\ref{fig:sat_example} shows the propositional encoding of the computation for the word $w=\texttt{aab}$ in the DFA $A_2$, which correspond to the first fourth constraints.
    The constraint 5 is not shown in the figure, but it is encoded as a disjunction of the final states of the computation, i.e., $\bigvee\limits_{f \in F} y_{k, f} = y_{3,q_{1}}\lor y_{3,q_{2}}$ for $A_2$.
    Observe that the computation with the input word $w$ force the final state to be $q_2$, which is not an accepting state, making the formula unsatisfiable.
    Finding the minimal unsatisfiable subset involves identifying the minimal indexes in the word that ensures the rejection in $A_2$. Since $A_2$ is a complement of $A_1$, ensuring rejection in $A_2$ also guarantees acceptance in $A_1$
\end{example}

\begin{figure}[ht]
    \centering
    \scalebox{0.7}{\begin{tikzpicture}[thick,node distance=2cm, every node/.style={scale=1.2}]
    \node (q1) [draw=none] {$q_{1}$};
    \node (q2) [right=4.5cm of q1, draw=none, circle] {$q_{2}$};
    \node (q3) [right=4.5cm of q2, draw=none, circle] {$q_{2}$};
    \node (q4) [right=4.5cm of q3, draw=none, circle] {$q_{3}$};

    \draw[->] (q1) -- (q2) node[midway, above] (x0) {$a$};
    \draw[->] (q2) -- (q3) node[midway, above] (x1) {$a$};
    \draw[->] (q3) -- (q4) node[midway, above] (x2) {$b$};

    \node[above=0.5cm of q1] (y_sat_1) {$y_{1,q_1} = 1$};
    \node[above=0.7cm of x0] {$x_{0,a} = 1$};
    \node[above=0.2cm of x0] (x_sat_0) {$x_{0,b} = 0$};
    \node[above=0.7cm of x1] {$x_{1,a} = 1$};
    \node[above=0.2cm of x1] (x_sat_1) {$x_{1,b} = 0$};
    \node[above=0.7cm of x2] {$x_{2,a} = 0$};
    \node[above=0.2cm of x2] (x_sat_2) {$x_{2,b} = 1$};

    \node[above=0.6cm of q2] (y_sat_2) {$y_{1,q_1} \land x_{1,a} \Rightarrow y_{1,q_2}$};
    \node[above=0.6cm of q3] (y_sat_3) {$y_{1,q_2} \land x_{2,a} \Rightarrow y_{2,q_2}$};
    \node[above=0.6cm of q4] (y_sat_4) {$y_{2,q_2} \land x_{2,b} \Rightarrow y_{3,q_3}$};

    \draw[dashed, blue, ->] (x_sat_0) -- (x0);
    \draw[dashed, blue, ->] (x_sat_1) -- (x1);
    \draw[dashed, blue, ->] (x_sat_2) -- (x2);

    \draw[dashed, blue, ->] (y_sat_1) -- (q1);

    \draw[dashed, blue, <-] (q2) -- ++(1.3, 1);
    \draw[dashed, blue, <-] (q3) -- ++(1.3, 1);
    \draw[dashed, blue, <-] (q4) -- ++(1.3, 1);

\end{tikzpicture}}
    \caption{Example of propositional encoding of a computation in a DFA.}
    \label{fig:sat_example}
\end{figure}
}

\section{Enumerating All Explanations}\label{sec:alg}

\ignore{
In this section we describe a novel approach to compute Formal Explanations
for Finite Automata and a given word. With this approach, we use a function
to check whether the language of a candidate explanation is a subset of the
language of the target FA being explained.
\pjs{Please use the macro $\any$ to refer to the any symbol, currently
  defined as $\Sigma$}
We adopt the symbol ``$.$" (dot) to represent any symbol in the alphabet $\Sigma$, as it is widely used in regular languages for this purpose.
}

While a single AXp/CXp extraction is quite efficient, for unbiased
explanations we want to extract all AXps.
Explanation
enumeration poses other challenges, since the number of AXps can be
exponential both in the size of the DFA $\mathcal{A}$ and in the
number of CXps.
That is a consequence of Proposition~\ref{thm:duality}.
Whenever one type of explanation (either AXp or CXp) is a collection
of disjoint sets, the collection of minimal hitting sets, i.e. dual
explanations (either CXp or AXp), is exponential (see appendix for a
concrete example).

%



Exhaustive enumeration of all AXps and CXps is achievable by
exploiting the minimal hitting set duality
between these explanations, similar to the
MARCO algorithm \cite{pms-aaai13,marco_enum_grow_shrink,
max_marco_enum} used in the context of over-constrained systems.
%
The process results in collecting all target and dual explanations.

Algorithm~\ref{alg:enum-adaptative} enumerates all explanations, both
AXps and CXps.
It can run in two modes: with $b = \mathit{true}$ it \emph{targets}
AXps with dual explanations being CXps, and if $b = \mathit{false}$ it
targets CXps with dual AXps.
Given an automaton $\mathcal{A}$ and word $w$ to explain, it
initialises the sets of target and dual explanations $\mathbb{E}_t$
and $\mathbb{E}_d$ as $\emptyset$.
It generates a new candidate target explanation $\mu$ as a minimal
hitting set of the dual sets $\mathbb{E}_d$ not already in or a
superset of something in $\mathbb{E}_t$.
This candidate generation can be made efficient using a Boolean
satisfiability (SAT) solver set to enumerate \emph{minimal} or
\emph{maximal
models}~\cite{giunchiglia-ecai06,marco_enum_grow_shrink}.
Initially, $\mu$ will be the empty set.
It then checks whether this candidate $\mu$ is a target explanation
using \textsc{IsTargetXP}.
Depending on the target type, it builds the regular language $W$
defined by $\mu$, by fixing or freeing $\mu$'s symbols, and checks the
\emph{language inclusion}.
The candidate $\mu$ is a target explanation if the check agrees with
$b$.
If this is the case, we simply add $\mu$ to
$\mathbb{E}_t$.
Otherwise, we extract a minimal subset $\nu$ of the
complement of $\mu$ using \textsc{ExtractDualXP}, which uses $b$ to
call the correct minimization procedure (either \textsc{ExtractCXp}
or \textsc{ExtractAXp}).
We then add counterexample $\nu$ to $\mathbb{E}_d$.
The loop continues until no new candidate hitting sets remain for
$\mathbb{E}_d$, in which case we have enumerated all target and dual
explanations.



\begin{algorithm}[tb]
  \caption{\textsc{XPEnum} -- Explanation Enumeration}
\label{alg:enum-adaptative}
\textbf{Input}: Automaton $\mathcal{A}$, word $w$, bounds $l$, $u$, AXp/CXp flag $b$ \\
\textbf{Output}: Explanation sets $\mathbb{E}_t$ (target) and $\mathbb{E}_d$ (dual)

\begin{algorithmic}[1]
  \STATE $(\mathbb{E}_t , \mathbb{E}_d) \gets (\emptyset, \emptyset)$
  \WHILE{true}
    \STATE $\mu \gets \textsc{MinimalHS}(\mathbb{E}_d, \mathbb{E}_t)$  \label{ln1:mhs}
    \IF{$\mu = \bot$}
      \STATE \textbf{break}
    \ENDIF
    \IF{\textsc{IsTargetXp}($\mu, \mathcal{A}, w, l, u, b$)} \label{ln1:check}
      \STATE $\mathbb{E}_t \gets \mathbb{E}_t \cup \{\mu\}$ \hfill \textit{// collect target explanation} \label{ln1:targrec}
    \ELSE
      \STATE $\nu \gets \textsc{ExtractDualXp}(\overline{\mu}, \mathcal{A}, w, l, u, b)$ \label{ln1:extract}
      \STATE $\mathbb{E}_d \gets \mathbb{E}_d \cup \{\nu\}$ \hfill \textit{// collect dual explanation} \label{ln1:dualrec}
    \ENDIF
  \ENDWHILE
  \STATE \textbf{return} $(\mathbb{E}_t, \mathbb{E}_d)$
\end{algorithmic}
\end{algorithm}

\begin{figure}[tb]
  \centering
  \begin{tikzpicture}[
  >-stealth,
  thick,
  axp/.style={fill=black, regular polygon, regular polygon sides=4, inner sep=1pt, minimum size=3pt},
  cxp/.style={fill=black, regular polygon, regular polygon sides=3, inner sep=1pt, minimum size=3pt},
  projection/.style={<->, dashed, shorten <=2pt, shorten >=2pt},
  groupbox/.style={draw, rectangle, rounded corners=4pt, inner sep=2pt}
]
  \node at (0,3) {Abductive Explanations};
  \node at (4.5,3) {Contrastive Explanations};

  \node[axp,label=left:1] (axp1) at (0,2.5) {};
  \node[axp,label=left:2] (axp2) at (0,2.1) {};
  \node at (0,1.6) {$L(\texttt{a} \Sigma\Sigma \texttt{c})$};

  \node[axp,label=left:1] (axq1) at (0,0.9) {};
  \node[axp,label=left:4] (axq4) at (0,0.5) {};
  \node at (0,0) {$L(\texttt{ac} \Sigma \Sigma)$};

  \node[cxp,label=right:1] (cxp1) at (4.5,2.3) {};
  \node at (4.5,1.9) {$L(\Sigma \texttt{ccc})$};

  \node[cxp,label=right:2] (cxp2) at (4.5,0.9) {};
  \node[cxp,label=right:4] (cxp4) at (4.5,0.5) {};
  \node at (4.5,0) {$L(\texttt{a} \Sigma \texttt{c} \Sigma)$};

  \node[coordinate] (padR1) at ([xshift=-0.4cm]axp1) {};
  \node[coordinate] (padR2) at ([xshift=-0.4cm]axq1) {};
  \node[coordinate] (padR3) at ([xshift=0.4cm]cxp1) {};
  \node[coordinate] (padR4) at ([xshift=0.4cm]cxp2) {};

  \node[groupbox, fit=(axp1)(axp2)(padR1), inner ysep=4pt] {};
  \node[groupbox, fit=(axq1)(axq4)(padR2), inner ysep=4pt] {};
  \node[groupbox, fit=(cxp1)(padR3), inner ysep=4pt] {};
  \node[groupbox, fit=(cxp2)(cxp4)(padR4), inner ysep=4pt] {};

  \draw[projection] (axp1) -- (cxp1);
  \draw[projection] (axq1) -- (cxp1);
  \draw[projection] (axp2) -- (cxp2);
  \draw[projection] (axq4) -- (cxp4);
\end{tikzpicture}
  \caption{
    Duality between AXps and CXps in $\lwa_1^1$.
    AXps fix the characters at given indices;
    CXps free them.
  }
  \label{fig:mhs_explanations_dot}
\end{figure}
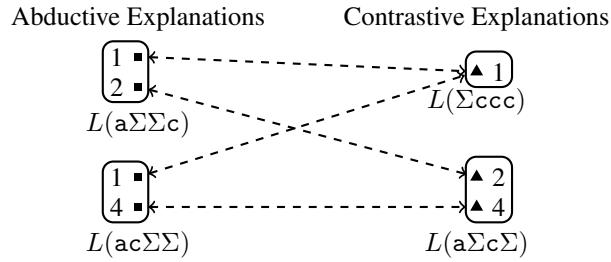
\begin{table}[tb]
\centering
\small 
\caption{Step-by-step enumerating all explanations targeting CXp for word
\texttt{accc} in the DFA illustrated in Figure~\ref{fig:dpi_axps} using $\lwa_1^1$.
Here, $\mu$ represents the Minimal Hitting Set (MHS) calculated at each
iteration.}
  \begin{tabular}{|llcccc|}
    \hline
    $\mathbb{E}_\text{CXp}$  & $\mathbb{E}_\text{AXp}$            & $\mu$ (\textsc{MHS})       & $\cxp_1^1(\mu, w)$ & \textsc{IsCXp?} & Result \\ \hline
    $\emptyset$              & $\emptyset$                        & $\emptyset$   & \texttt{accc}    & \textsf{false}    & ExtractAXp: \texttt{a}$\any\any$\texttt{c} $\{1, 4\}$ \\ 
    $\emptyset$              & $\{1, 4\}$           & $\{1\}$     & $\any$\texttt{ccc}               & \textsf{true}     & $\mathbb{E}_\text{CXp} \gets \mathbb{E}_\text{CXp} \cup \mu$  \\
    $\{1\}$                  & $\{1, 4\}$           & $\{4\}$     & \texttt{acc}$\any$               & \textsf{false}    & ExtractAXp: \texttt{ac}$\any\any$ $\{1, 2\}$ \\
    $\{1\}$                  & $\{1, 4\}, \{1, 2\}$ & $\{2, 4\}$  & \texttt{a}$\any$\texttt{c}$\any$ & \textsf{true}     & $\mathbb{E}_\text{CXp} \gets \mathbb{E}_\text{CXp} \cup \mu$  \\
    $\{1\}, \{2, 4\}$        & $\{1, 4\}, \{1, 2\}$ & $\bot$      & \textbf{break}                   & ---               & ---                   \\
    \hline
  \end{tabular}

\label{tab:enumeration_step_by_step}
\end{table}
\begin{example}
  Table~\ref{tab:enumeration_step_by_step} traces the execution of
  Algorithm~\ref{alg:enum-adaptative} targeting CXps ($b =
  \mathit{false}$) for the DFA in Figure~\ref{fig:dpi_axps} and
  word $w=\texttt{accc}$ from Example~\ref{example_dot}.
  Figure~\ref{fig:mhs_explanations_dot} illustrates the duality
  relationship between the AXps and CXps.
\end{example}
\noffa{
In the context of explanation enumeration, we can define a formal feature
attribution for each position in the string, explaining how involved it is
in generating the acceptance.
\begin{example}[FFA]
  Consider the DFA and rejected word $w=\texttt{accc}$ illustrated in
  Figure~\ref{fig:dpi_axps}, with its AXps shown in
  Example~\ref{ex:allaxps}.
  We can generate a feature attribution for $w$'s rejection, from the
  AXps $\axp_1^1(\{1,2\},w)$ and $\axp_1^1(\{1,4\},w)$.
  The FFA is $\{ (1, 1), (2, 0.5), (4, 0.5)\}$, indicating that
  position 1 is the most important position for the rejection.
  The next most important positions are 2 and 4.
  And the position 3 does not play any role.
  \qed
\end{example}

While previous example shows the calculation of FFAs, the following
example illustrates how FFA can provide an intuition easily
representable graphically on why some of the decisions are critical in
rejecting a word by a DFA.

\begin{example}
  \label{ex:maze} Consider a language where words represent paths
  through a maze. A word is accepted if the path reaches the ``EXIT''
  cell without hitting walls. 
  Figure~\ref{fig:maze_experiments} exemplifies the computation of
  formal feature attribution (FFA) for the decisions made a in grid
  maze.
  In particular, Figure~\ref{fig:maze_rejected_path} shows one AXp of
  minimal cardinality (size one) corresponding to the final right
  arrow, which indicates that any path of length 13 ending with this
  symbol is rejected by the DFA using the explanation language
  $\lwa_1^1$. 
  Table~\ref{tab:enum_AXps_star} then shows all AXps and computes the
  FFA for each input symbol / decision.
  Finally, Figure~\ref{fig:maze_ffa} visually depicts the resulting
  attributions, where arrows represent the relative importance of
  specific decisions. 
  The most important symbol leading to rejection occurs at position
  11, the symbol at this step has the highest occurrence in the AXps,
  and thus the highest attribution. Considering the language
  $\lwa_1^1$, which fixes the word length, making the symbol ``up'' a
  bad decision when there are only two remaining symbols.
  \qed
\end{example}

\begin{figure}[t]
  \begin{subfigure}[b]{0.325\columnwidth}
    \centering
    \includegraphics[height=4.5cm]{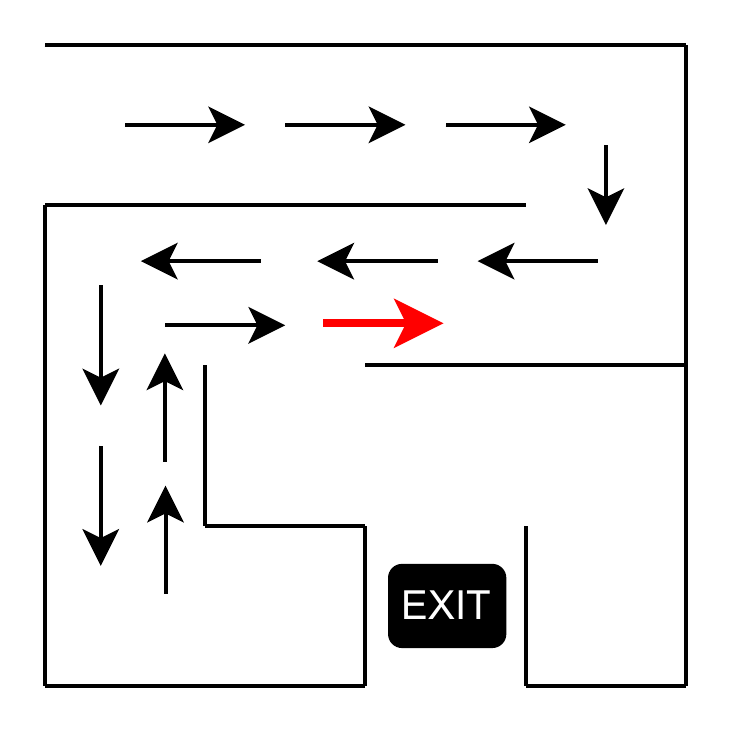}
    \caption{Example maze with a rejected path. An AXp is highlighted in red, showing that every word ending in $\rightarrow$ is rejected.}
    \label{fig:maze_rejected_path}
  \end{subfigure}
  \hfill
  \begin{subfigure}[b]{0.325\columnwidth}
    \centering
    \renewcommand{\arraystretch}{0.55}
    \begin{tabular}{lll} 
        \toprule
        \multicolumn{3}{l}{\textbf{AXps}} \\ \midrule
        \multicolumn{3}{p{0.9\linewidth}}{$\{13\}$, $\{8, 9\}$, $\{8, 10\}$, $\{10, 11\}$, $\{9, 11\}$, $\{7, 11\}$, $\{11, 12\}$, $\{8, 11\}$} \\ \midrule
        \textbf{Index} & \textbf{Freq.} & \textbf{FFA ($\frac{Freq.}{|AXps|}$)} \\ \midrule
        11 & \textbf{5} & 0.625 \includegraphics[width=0.7cm, valign=c]{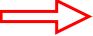} \\ 
        8 & 3 & 0.375 \includegraphics[width=0.8cm, valign=c]{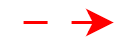} \\ 
        9, 10 & 2 & 0.250  \includegraphics[width=0.8cm, valign=c]{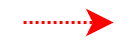} \\ 
        7, 12, 13 & 1 & 0.125  \includegraphics[width=0.8cm, valign=c]{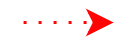} \\ \bottomrule
    \end{tabular}
    \vspace{0.23cm}
    \caption{All AXps in language $\lwa_1^1$ and FFA. Each FFA value is represented with an arrow of different shape.}
    \label{tab:enum_AXps_star}
  \end{subfigure}
  \hfill
  \begin{subfigure}[b]{0.325\columnwidth}
    \centering
    \includegraphics[height=4.5cm]{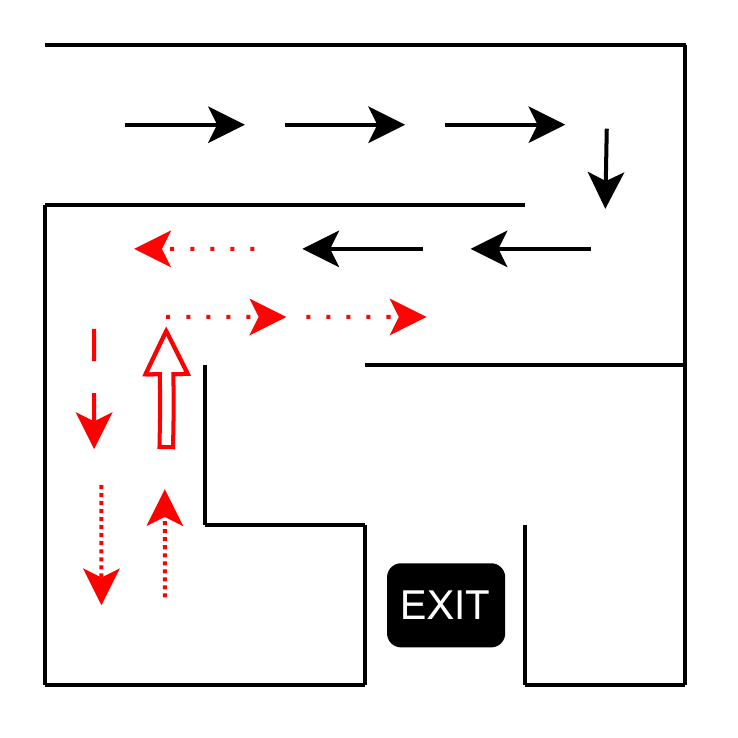}
    \caption{Illustration of FFAs in the maze. The arrows indicate the importance of each position in the path for rejection.}
    \label{fig:maze_ffa}
  \end{subfigure}

  \caption{FFA for maze using language $\lwa_1^1$.}
  \label{fig:maze_experiments}
\end{figure}

}

\subsection{Extra Heuristics.}
In practice, changing a single symbol in $w\in L(\mathcal{A})$ often
yields a counterexample, i.e. many CXps are singletons.
We can determine all singleton CXps for $w \in L(\mathcal{A})$ in
$\lwa_l^u$ by checking for each $i \in \{1,\ldots,|w|\}$ whether
$\cxp_l^u(\{i\},w) \not\subseteq L(\mathcal{A})$.
If this holds then $\cxp_l^u(\{i\},w)$ is a CXp.
If we run Algorithm~\ref{alg:enum-adaptative} targeting AXps ($b =
\mathit{true}$) then we can \emph{warm-start} the algorithm by
replacing $\mathbb{E}_d$ by the set of all singletons $\{i\}$ where
$\cxp_l^u(\{i\},w)$ is a CXp.
This approach reduces the number of iterations of
Algorithm~\ref{alg:enum-adaptative} and we will use it in the
experiments to speed-up AXp computation.
Another modification we can make to
Algorithm~\ref{alg:enum-adaptative} is to replace \textsc{MinimalHS},
with \textsc{MinimumHS}, which finds a \emph{minimum size} hitting set
not already explored.
This can be done by using, e.g., an efficient integer linear
programming (ILP) solver~\cite{gurobi} or a maximum satisfiability
(MaxSAT) solver~\cite{imms-jsat19}.
This will guarantee that the first target explanation we find is the
smallest possible.

\section{Experiments}\label{sec:exp}

Given that the proposed explanations build on the definitions grounded
in formal logic, the purpose of our experiments is to determine how
performant they are in practice rather than to test how meaningful
they are to humans.
Also, as argued in Section~\ref{sec:prelim}, it is practically
challenging to select a \emph{best} feature selection explanation
(either AXp or CXp). 
This can be tackled by determining how \emph{important} individual
symbols in the input are for a given decision of a DFA.
A way to do that, formal feature attribution (FFA), requires one to
enumerate explanations.
Therefore, the primary objective of our experimental evaluation is to
assess the scalability of exhaustive explanation enumeration.
Note that while computing a single minimum edit distance explanation
requires polynomial time and so can be done efficiently, it struggles
with complete explanation enumeration.
FFA computation is illustrated with the Maze benchmarks in
Example~\ref{ex:maze}, where the parts of inputs with high FFA
represent the most critical decision point in a path.


Overall, we test Algorithm~\ref{alg:enum-adaptative} on three
benchmark families:
(1)~Deep Packet Inspection (DPI) rules,
(2)~a generated corpus of random FA,
(3)~DNA sequences containing a known motif, and
(4)~Mazes represented as DFA.
%
%
%
All the experiments were run on Ubuntu 20.04 LTS with an Intel Xeon
8260 CPU and 16 GB RAM.
The implementation builds on the PySAT
toolkit~\cite{imms-sat18,itk-sat24} to enumerate minimal hitting sets
and the library Mata~\cite{Mata_lib} for language inclusion checks.
Each run had a 600 seconds timeout.
Hereinafter, we focus on $\lwa_1^1$.
We compare three different approaches for computing all explanations:
with the flag set to \emph{true}, targeting AXps; with the flag set to
\emph{true}, initialized with all singleton CXps; and with the flag
set to \emph{false} targeting CXps.
All the plots are shown as \emph{cactus plots} sorting results by
runtime to compare overall effectiveness of the approaches.

\newcommand{\mysubsection}[1]{\paragraph{#1}}

\begin{wrapfigure}[31]{R}{0.55\textwidth}
\centering
    \includegraphics[width=0.55\textwidth]{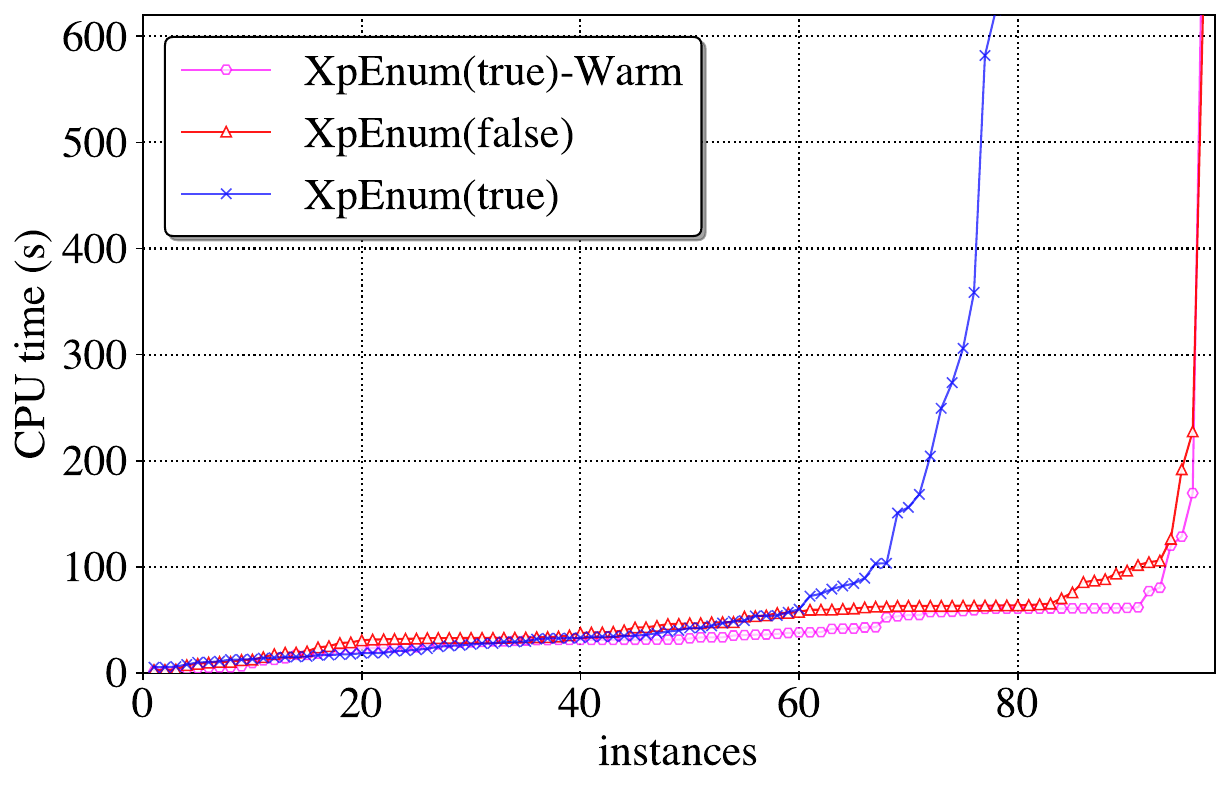}
    \vspace{-20pt}
    \caption{Deep Packet Inspection performance.}
    \label{fig:cactus_packet}
    \vspace{10pt}
    \includegraphics[width=0.55\textwidth]{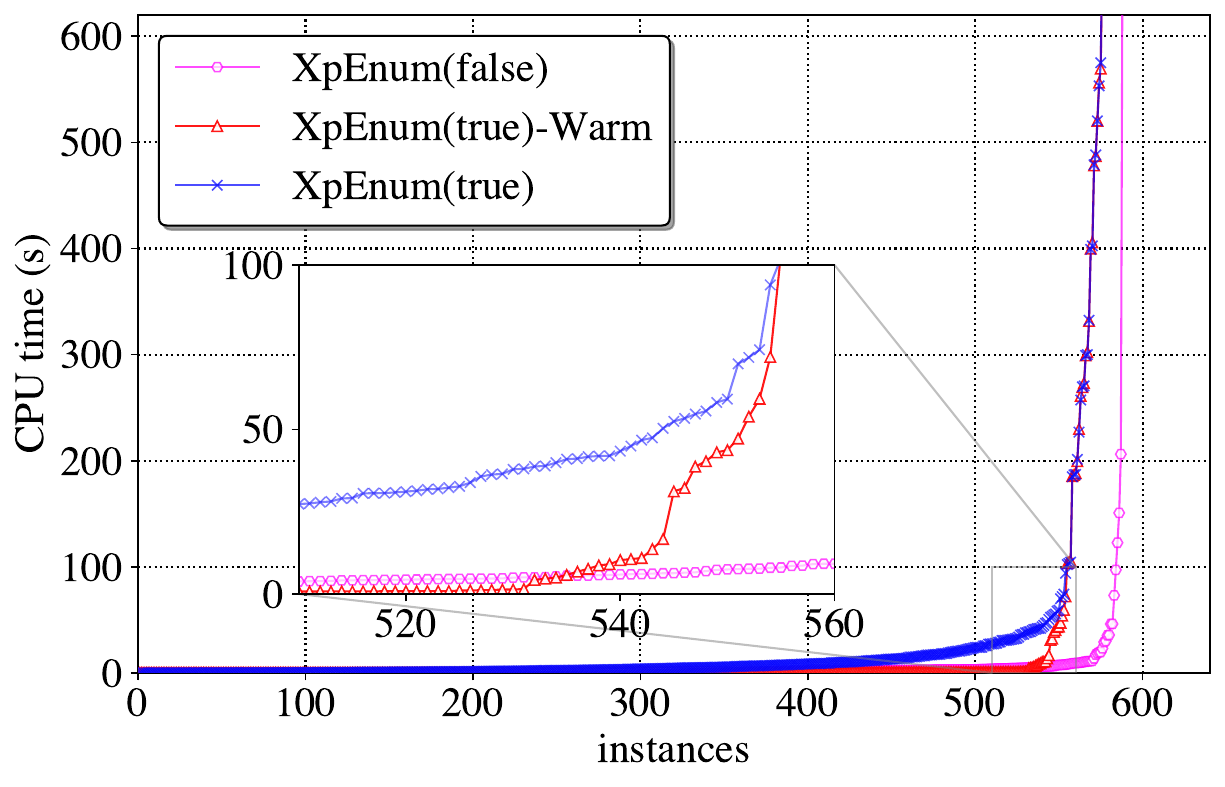}
    \vspace{-20pt}
    \caption{Generated FA corpus, with \emph{accepted} words.}
    \label{fig:cactus_corpus_accepted}
\end{wrapfigure}

\mysubsection{Deep Packet Inspection.}
Deep Packet Inspection identifies malicious traffic by
matching packet contents to known signature
patterns, often represented as regular
expressions. 

We use a dataset of 98 File Transfer Protocol (FTP) signatures from
\cite{packet_inspection}.\footnote{\url{https://pages.cs.wisc.edu/~vg/papers/raid2010/data/ftp-98.txt}}
The task is to check
whether a given word matches any of these regular expressions.

We created an FA with the union of the 98 regular expressions,
resulting in an (non-deterministic) FA with 30,590 states and
2,752,238 transitions.\footnote{The automata package we used cannot convert it to a DFA}

This large FA is used to test scalability of Algorithm~\ref{alg:enum-adaptative}.
The length of the input words range from 4 to 2557,
and the average length is 156.

Figure~\ref{fig:cactus_packet} depicts the performance of the three
modes of the Algorithm~\ref{alg:enum-adaptative}.
Observe that the algorithm scales with large automata, where the warm-starting
makes a huge difference, while targeting AXps alone is worse than
targeting CXps.

\mysubsection{Generated corpus of FA.}
We generated a corpus of FAs to test explanation enumeration for long
words matching a set of possible words.
This corpus of FAs is similar to the corpus used in the work
\cite{fixing_budget} in a different context.
First, we generate $m$ random words with length $l$ over the alphabet
$\{1,\ldots,d\}$ for $l\in \{5,10,15,20\}$, $m\in \{1,3,5,10\}$, $d\in
\{2,3,5,10\}$.
Then, we build an FA for each configuration such that it accepts the
language $\{\Sigma^* w \Sigma^* \mid w \in M\}$.
For each of the 64 resulting FAs, we generated 10 accepted and 10
rejected words of length~$\{i \times 100 \mid i \in \{1,\ldots,10\}\}$.

Figure~\ref{fig:cactus_corpus_accepted} shows how the three
alternatives perform on accepted words. The warm-started enumeration focusing on AXps is usually the best
method, but for the hardest instances, targeting CXps is more
effective due to the large number of CXps.
\mysubsection{DNA Sequence.}
\begin{wrapfigure}[16]{R}{0.55\textwidth}
\centering
    \vspace{-10pt}
    \includegraphics[width=0.55\textwidth]{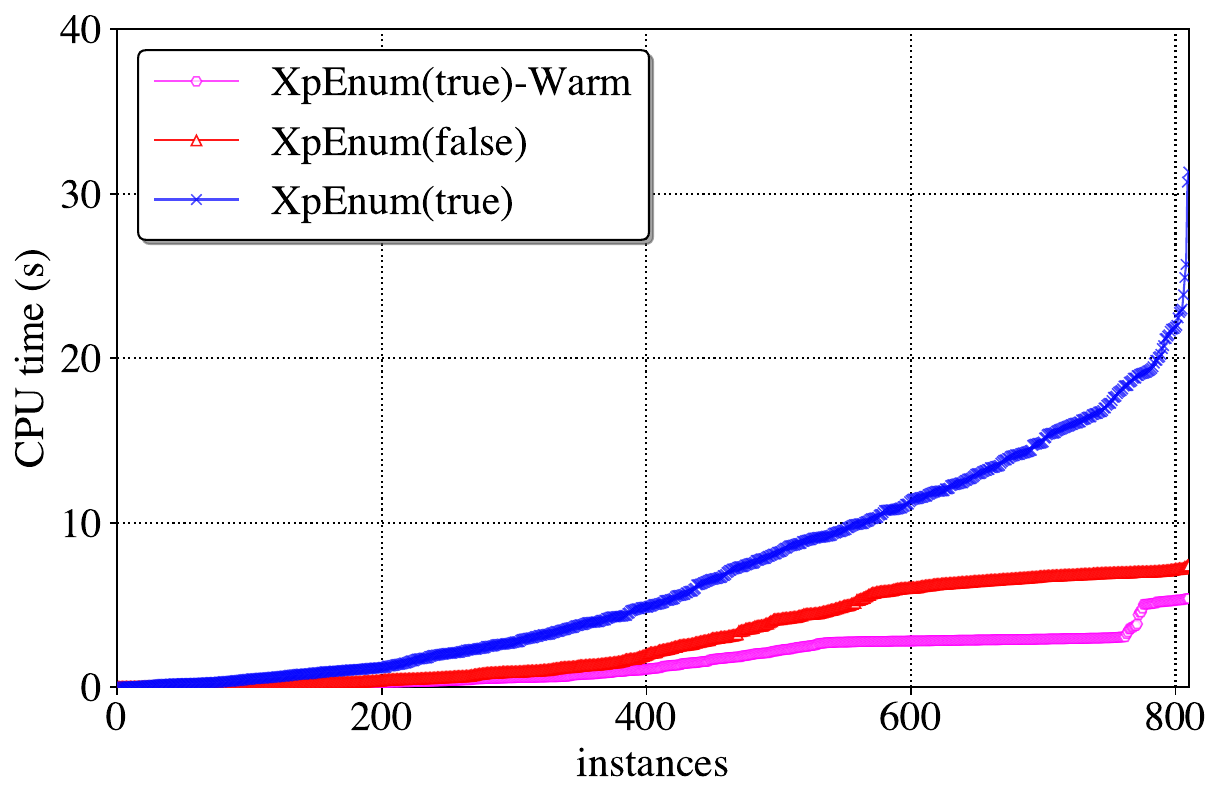}
    \vspace{-20pt}
    \caption{Motifs in DNA sequences.}
    \label{fig:cactus_dna}
\end{wrapfigure}
DNA is a molecule conformed by two complementary chains of
nucleotides. These nucleotides contain four types: adenine (A),
thymine (T), cytosine (C), and guanine (G).
DNA sequences contain motifs, which are recurring patterns that are
deemed to have a biological function~\cite{dna_motif}.

We used a benchmark of 25 datasets of real DNA sequences with a known
motif~\cite{dna_motifs_benchmark}.\footnote{Available at \url{https://tare.medisin.ntnu.no/pages/tools.php}.}
Each dataset has multiple pairs of $(\text{DNA sequence},
\text{motif})$; in total, there are 810 sequences.
This case is used to test the ability to find AXps for words with
different lengths.
The shortest sequence has 62 nucleotides, the longest 2000 and the
average length is 1246.

Figure~\ref{fig:cactus_dna} shows once more that
targeting AXps with warm-start is the superior approach, and without
warm-start, targeting CXps is better than targeting AXps.

\mysubsection{Maze.}

\begin{wrapfigure}[14]{R}{0.55\textwidth}
\centering
    \includegraphics[width=0.55\textwidth]{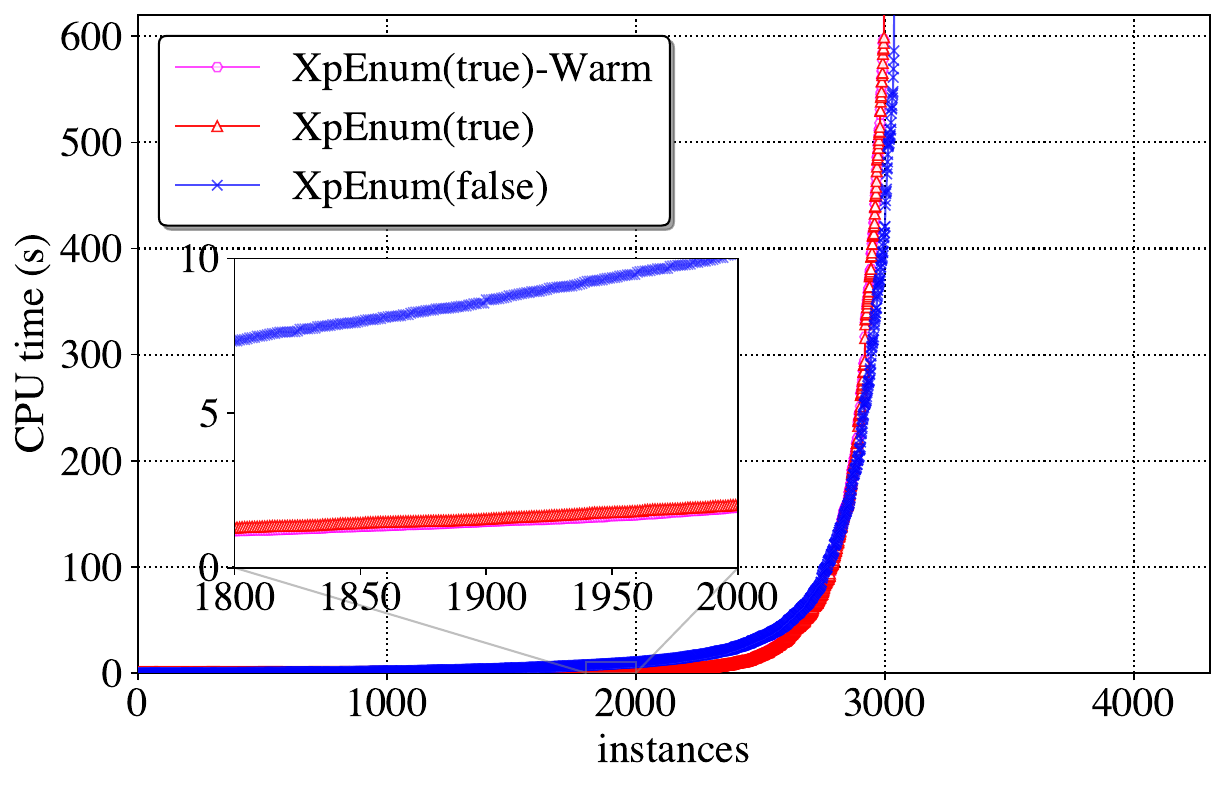}
    \vspace{-20pt}
    \caption{Rejected paths in mazes.}
    \label{fig:cactus_maze}
\end{wrapfigure}
This benchmark set is based on the Example~\ref{ex:maze},
we generate 861 mazes with dimensions
ranging from $10\times 10$ to $50\times 50$. 
For each height $h \in \{10,\dots,50\}$, we produce mazes of size
$h\times w$ for all $w \in \{h,\dots,50\}$. 
This results in a total of $\sum_{h=10}^50 (51 - h) = 861$ distinct
maze sizes.

Walls were randomly placed with a 1/3 probability per cell.
To generate a rejected path, we randomly select a path from the start
to the end of the maze and introduce between 1 and 5 random changes to
the path, ensuring that the modified path is still rejected by the
maze. 
If the maze does not have a solution, we repeat the process until we
find a solvable maze for that dimension.

Figure~\ref{fig:cactus_maze} evaluates complete explanation
enumeration, necessary for defining the FFA.
Given that this is a complex task, $\approx1000$ problem instances
reach the 600s time out.
Targeting AXps is preferable except for the hardest instances. 
The warm start does not improve performance here. 

\mysubsection{Summary.}
Overall, our results indicate that our approach to enumerate AXps and
CXps for finite automata is scalable.
In particular, we are able to solve the problem for the automata with
the size varying from 10 to 2,500, excluding the Deep Packet
Inspection scenario, which has many more states.
The average number of AXps and CXps is 50 and 255, respectively.
The average number of unit-size CXps is 8.
The average size of an AXp is 11 while the average size of a CXp is
15.
Based on the instances where the enumeration process was completed,
the average completion times were 23.97, 21.59, and 19.36 seconds for
the Algorithm \ref{alg:enum-adaptative} with flags \emph{false},
\emph{true}, and \emph{true} initialized with all singleton CXps,
respectively.
This shows better performance when targeting AXps, and warm-starting
with all singleton CXps.
However, harder cases favor targeting CXps directly.

Finally, the general observations made here for $\lwa_1^1$ also apply
to $\lwa_1^\infty$ and
$\lwa_0^\infty$ (results can be found in the appendix).
Importantly, as suggested by Propositions~\ref{rel-cxp}--\ref{rel-axp},
the structure and thus the number of explanations for $\lwa_1^\infty$
and $\lwa_0^\infty$ affect the performance of the enumeration
and highlight an advantage of targeting CXps.

\section{Conclusions and Related Work}\label{sec:conc}

We introduce a first approach to computing formal explanations for
finite automata (FA) decisions $w \in L(\mathcal{A})$.
We show we can usually efficiently compute a single AXp, even for
large automata and long words.
And for many problems we can enumerate all explanations, which enables
us to determine formal feature attribution (FFA) for all the input
symbols.

The results of our work are applicable to problems that often rely on
finite automata, including lexical analysis, string searching, and
pattern matching.
While AXps can be employed to understand why certain acceptance or
rejection decisions have been made, CXps can be used to suggest how
those decisions can be altered.
Furthermore, a significant advantage of our approach is that these
formal explanations are verifiable. 
The proposed AXps and CXps provide logical guarantees that can be
validated through formal methods, such as model checking.
Finally, FFA can be applied to reveal the importance of certain
features of the input words that are otherwise hidden.

\bibliographystyle{plainurl}
\bibliography{explaining-fa-2026}

\appendix

\clearpage

\section{Supplementary Material}

\subsection{Propositions and proofs}

This section provides the complete set of formal
propositions and their corresponding proofs.

\setcounter{proposition}{0}

\begin{proposition}
    \label{thm:duality_appendix}
    Given $w\in L(\mathcal{A})$,
    assume that the sets of all AXps and CXps are
    denoted as $\mathbb{A}=\{X\subseteq\{1,\ldots,|w|\} \mid 
    \axp_l^u(X, w)\in\mathbb{W}\}$ and
    $\mathbb{C}=\{Y\subseteq\{1,\ldots,|w|\} \mid \cxp_l^u(Y,
    w)\in\mathbb{W}\}$, respectively.
    It holds that each $X\in\mathbb{A}$ is a
    minimal hitting set of $\mathbb{C}$ and, vice versa, each
    $Y\in\mathbb{C}$ is minimal hitting set of
    $\mathbb{A}$.
\end{proposition}
\begin{proof}
    Using the explanation language $\lwa$, we can treat
    the symbols of input $w$ as \emph{features}, and
    $\axp_l^u(S,w)$ is an AXp iff $S$ is an AXp of the classification
    problem $\kappa(w) = \mbox{``accept''}$ s.t.\ $\kappa$ is the DFA
    applied to $w$.
    Similarly $\cxp_l^u(S,w)$ is a CXp iff $S$ is a CXp of the
    classification problem $\kappa(w) = \mbox{``accept''}$.
    Hence, we directly have the hitting set duality between AXps
    and CXps of $w \in L(\mathcal{A})$.
\end{proof}

\begin{proposition}\label{rel-axp_appendix}
    Let $w\in L(\mathcal{A})$ and $l_1..u_1 \supseteq l_2..u_2$.
    If $S$ is such that
    $\axp_{l_1}^{u_1}(S,w)$ is a weak AXp for $w \in L(\mathcal{A})$ using
    $\lwa_{l_1}^{u_1}$
    then $\axp_{l_2}^{u_2}(S, w)$ is a weak AXp using $\lwa_{l_2}^{u_2}$.
\end{proposition}

\begin{proof}
    The relation between the intervals  implies that
    $\Sigma_{l_1}^{u_1}\supseteq~\Sigma_{l_1}^{u_2}$.
    Note that a larger set of replacements yields a larger set of languages
    for any $S$ subset of $\{1, \dots \vert w \vert\}$,
    and we have that
    $\axp_{l_1}^{u_1}(S, w) \supseteq \axp_{l_2}^{u_2}(S, w)$.

    $\axp_{l_1}^{u_1}(S, w)$ is a weak AXp, so
    $\axp_{l_1}^{u_1}(S, w)\subseteq L(\mathcal{A})$ and hence
    $\axp_{l_2}^{u_2}(S, w) \subseteq L(\mathcal{A})$. Thus, $\axp_{l_2}^{u_2}(S,
    w)$ is a weak AXp for $w \in L(\mathcal{A})$ using $\lwa_{l_2}^{u_2}$.
\end{proof}

\begin{proposition}\label{rel-cxp_appendix}
    Let $w\in L(\mathcal{A})$, and $l_1..u_1\subseteq l_2..u_2$.
    If $S$ is such that
    $\cxp_{l_1}^{u_1}(S,w)$ is a weak CXp for $w \in L(\mathcal{A})$ using
    $\lwa_{l_1}^{u_1}$
    then $\cxp_{l_2}^{u_2}(S, w)$ is a weak CXp using $\lwa_{l_2}^{u_2}$.
\end{proposition}

\begin{proof}
    By the definition of a CXp, $\cxp_{l_1}^{u_1}(S, w)\not\subseteq L(\mathcal{A})$,
    there exists a word $w'\in \cxp_{l_1}^{u_1}(S, w)$, such that
    $w'\not\in L(\mathcal{A})$.
    Since $w'\in \cxp_{l_1}^{u_1}(S, w)$, each position $i$ in $S$
    is replaced in $w'$ by a string of length between $l_1$ and $u_1$.
    As $(l_1,u_1)\subseteq (l_2,u_2)$, the replacementsare also valid
    under bound $(l_2,u_2)$, to have $\cxp_{l_2}^{u_2}(S, w)$. Thus,
    $\cxp_{l_2}^{u_2}(S, w)$ is a CXp, since it contains $w$ and
    a counterexample $w'\not\in L(\mathcal{A})$.
\end{proof}

\begin{proposition} \label{prop:edfa_appendix}
  Given an alphabet $\Sigma$ and a parameter $n\in\mathbb{N}$, there
  exists a family of regular languages $L_n$ over $\Sigma$ such that
  any DFA $\mathcal{A}_n$ for $L_n$ has least $2^n$ states while $\forall
  w\in\Sigma^*$ for any $\axp_1^1(S, w)$ and any $\cxp_1^1(S, w)$ it
  holds that $|S|\leq k$, for some constant $k\in\mathbb{N}$, $k<n$.
\end{proposition}
\begin{proof}
  Let $\Sigma=\{\texttt{a},\texttt{b}\}$ and
  $R_n=\texttt{(a|b)}^*\texttt{a(a|b)}^{n-1}$ defining
  the language of strings over $\Sigma$ whose $n^{\text{th}}$ symbol
  from the end is \texttt{a}.
  It is well-known that any DFA $\mathcal{A}_n$ for $L_n$ has at least $2^n$
  states~\cite{Dragon_book,sipser-book97}.
  Now, consider any string $w\in \Sigma^*$.
  Observe that $w\in L_n$ if and only if the $n^{\text{th}}$ character
  from the end of $w$ is \texttt{a}.
  Hence, a sole AXp using $\lwa_1^1$ for this fact has size 1; namely,
  it is $\axp_1^1(\{|w|-n+1\}, w)$.
  (In other words, as long as we keep the symbol at position $|w|-n+1$
  unchanged, the outcome of the DFA $\mathcal{A}_n$ will remain the same.)
  Similarly, changing this particular character is the only way to
  change the result of the DFA, i.e., $\cxp_1^1(\{|w|-n+1\},w)$ is the
  only CXp for~$w$.
\end{proof}

\begin{proposition} \label{prop:ed_appendix}
  Given an alphabet $\Sigma$ and a parameter $n\in\mathbb{N}$, there
  exists a family of regular languages $L_n$ over $\Sigma$ and a word
  $w\not\in L_n$ with the minimum edit distance of at least $n$ such
  that for any AXp $\axp_1^1(S,w)$ it holds that $|S|\leq k$, for some
  constant $k\in\mathbb{N}$, $k<n$.
\end{proposition}
\begin{proof}
  Let $\Sigma=\{\texttt{a},\texttt{b}\}$ and
  $R_n=\texttt{(a|b)}^*\texttt{a}^{n}\texttt{(a|b)}^*$ defining
  the language of strings containing $n$ consecutive symbols \texttt{a}'s.
  Consider a string $w=\texttt{bb}\ldots\texttt{b}$ of size $m=~2n$.
  Clearly, $w\not\in L_n$, and the minimum edit distance from $w$ to
  $L_n$ is $n$, since at least $n$ consecutive
  characters in $w$ must be changed to \texttt{a}'s
  (starting from any suitable position)
  to obtain a word $w'\in L_n$.
  Observe that a possible AXp for $w$ using $\lwa_1^1$ is
  $\axp_1^1(\{1,n\},w)$.
  Indeed, as long as the two symbols at the beginning and the
  middle of $w$ are both \texttt{b}, there is \emph{no way} to obtain a word
  $w'$ with $n$ consecutive \texttt{a}'s.
  In fact, the complete set of AXps for $w$ is $\axp_1^1(\{i,j\},w)$
  for any $i,j\in\mathbb{N}$ such that $i<n$ and $n\le j<i+n$.
  Each such AXp has size 2.
  Similar reasoning applies more generally: when $w=\texttt{bb\ldots b}$ has
  length $m=kn$, all the AXps of $w$ are of size $k$.
\end{proof}

\begin{proposition}\label{thm:axp_appendix}
  Given a DFA $\mathcal{A}$ with $m$ states and a word $w$ s.t. $\lvert
  w\rvert=n$, computing an AXp for $w \in L(\mathcal{A})$ in languages
  $\lwa_1^1$, $\lwa_{1}^\infty$, or $\lwa_{0}^\infty$ can be done in
  $O(\lvert\Sigma\rvert mn^2)$ time.
\end{proposition}
\begin{proof}
  The call \textsc{ExtractAXp}$(\{1,\ldots,|w|\},\mathcal{A},w,l,u)$ always
  generates a minimal AXp of $w \in L(\mathcal{A})$ in language $\lwa_l^u$.
  This follows since $\axp_l^u(\{1,\ldots,|w|\}, w)$ is just a weak
  AXp $L(w)$, and the resulting set $X$ is minimal by construction,
  since adding back any removed position results in a weak AXp, which
  is strictly larger than the AXp without the position.
  To extract a minimal AXp for $w \in L(\mathcal{A})$, the \textsc{ExtractAXp}
  procedure examines exactly $n$ possible AXp candidates to find a
  minimal one.
  A candidate $\axp_l^u(X,w)$ is kept as a DFA $\mathcal{A}_w$ recognising the
  corresponding language, which is initially set to recognise $L(w)$
  (and so $\mathcal{A}_w$ initially has $n+1$ states) and it is updated
  \emph{incrementally} at each iteration.
  Assuming the language $\mathbb{W}_1^1$, each iteration of
  Algorithm~\ref{alg:extract_axp} involves the following operations:
  \begin{enumerate}[noitemsep]
    \item replacing a $w[i]$-transition in $\mathcal{A}_w$ with a
      $\Sigma$-transition, in $O(1)$ time;
    \item replacing a $\Sigma$-transition in $\mathcal{A}_w$ with a
      $w[i]$-transition, in $O(1)$ time (if the check determines that
      position $i$ must be kept);
    \item checking whether $\axp_l^u(X\setminus\{i\},w)\triangleq
      L(\mathcal{A}_w)\subseteq L(\mathcal{A})$, in $O(|\Sigma|mn)$
      time because $|\mathcal{A}|=m$
      and $|\mathcal{A}_w|=n+1$ (the above operations \emph{do not add} new
      states in $\mathcal{A}_w$).
  \end{enumerate}
  %
  %
  Hence, the overall complexity of the algorithm is $O(|\Sigma|mn^2)$
  for the language $\lwa_1^1$.\footnote{If we consider a more general
  language $\lwa_1^\infty$ ($\lwa_0^\infty$, resp.) then replacing a
  $w[i]$-transition in the corresponding DFA $A_w$ with a
  $\Sigma_1^\infty$-transition ($\Sigma_0^\infty$-transition, resp.)
  adds a single new state to $A_w$ in the worst case. This way, $A_w$
  still has $O(n)$ states at each iteration of the algorithm.
  Therefore, the complexity result holds.}
\end{proof}

\begin{proposition}
  Computing a single AXp using $\lwa_0^\infty$ for $w \in L(R)$ given
  a regular expression $R$ is PSPACE-hard.
\end{proposition}
\begin{proof}
  Note that $\emptyset$ defines an $\axp_0^\infty(\emptyset, w)$ for
  $w \in L(R)$ iff $R$ is a \emph{universal language}.
  If it is the case then $\axp_0^\infty(\emptyset, w)$ is
  \emph{unique}, due to minimality.
  An algorithm for determining an AXp for $w \in L(R)$ in
  $\lwa_0^\infty$ thus gives a check that $L(R)$ is universal, which
  is PSPACE-complete~\cite{universal}.
\end{proof}

\section{Exponential Size of Explanations}

Each type of explanation AXp/CXp may be exponential in the other
type of explanation.
Whenever one
type of explanation AXp/CXp is a collection of disjoint
sets, the collection of minimal hitting sets (dual
explanation) is exponential.

\begin{proposition} \label{thm:exponential_appendix}
    Given an finite automaton $\mathcal{A}$, the number of AXps/CXps
    for a word $w \in L(\mathcal{A})$, computed using $\lwa_1^1$,
    may be exponential in both the
    size of the $\mathcal{A}$ and the number of its dual explanations.
\end{proposition}

\begin{proof}
  Let $\mathcal{A}$ be an FA over the alphabet $\Sigma =
  \{\texttt{a},\texttt{b}, \dots , \texttt{z}\}\cup
  \{\texttt{A},\texttt{B}, \dots , \texttt{Z}\}$, where $\lvert
  \any\rvert = n$.
  The language $L(\mathcal{A})$ recognised by $\mathcal{A}$ is defined as:

  \begin{equation*}
      L(\mathcal{A})=\texttt{ab}\any^{\frac{n}{2}-2}\cup \any^{2}
      \texttt{cd}\any^{\frac{n}{2}-4}\cup \any^{4}
      \texttt{ef}\any^{\frac{n}{2}-6} \cup ... \cup
      \any^{\frac{n}{2}-2} \texttt{yz}
  \end{equation*}
  Each term in the union represents a language where a fixed
  adjacent pair of symbols are lowercase and appears at specific
  positions while the remaining positions can be any symbol from
  $\Sigma$.

  If we consider the word $w$ with all symbols in uppercase defined
  as: $w=\texttt{ABCD\dots YZ}$,  with $\lvert w \rvert = \lfloor
  \frac{n}{2} \rfloor$ then $w\not\in L(\mathcal{A})$ because the considered
  word does not match any fixed adjacent pair of lowercase symbols in
  $L(\mathcal{A})$.
  The set of minimal CXps using $\lwa_1^1$ for $w$ is: $\mathbb{C}=\{ \cxp_1^1(\{1,2\},w),
  \cxp_1^1(\{3,4\},w), \ldots,$  $\cxp_1^1(\{\lfloor \frac{n}{2} \rfloor-1,
  \lfloor \frac{n}{2} \rfloor\},w) \}$.
  \begin{itemize}
      \item $\{1,2\}$ is a CXp because $w'=abCD\dots YZ$ is in $L(\mathcal{A})$ and $w\not\in L(\mathcal{A})$.
      \item Similarly $\{3,4\}$ is a CXp because $w''=ABcd\dots YZ$ is in $L(\mathcal{A})$ and $w\not\in L(\mathcal{A})$.
      \item The same reasoning holds for each pair in $\mathbb{C}$.
  \end{itemize}
  The number of CXps for $w$ is $m={\lfloor \frac{n}{2}
    \rfloor}\times \frac{1}{2}$ and all CXps are disjoint sets.
  The set of AXps is defined by the all minimal hitting sets of $\mathbb{C}$
  due to the minimal hitting set duality,
  (see Theorem~\ref{thm:duality}), it is of size $2^m$,
  exponentially larger than the number of CXps ($m$) and
  the size of $\mathcal{A}$
  (as an NFA) $O(m^2)$.

  Similarly, if we consider the word $w_2$ with all symbols in lowercase
  defined as: $w_2=\texttt{abcd\dots yz}$, then $w_2\in L(\mathcal{A})$
  because the considered word matches at least a fixed adjacent
  pair of lowercase symbols in $L(\mathcal{A})$ (matches all fixed adjacent pairs).
  The set of minimal AXps using $\lwa_1^1$ for $w_2$ is: $\mathbb{A}=~\{ \axp_1^1(\{1,2\},w_2),
  \axp_1^1(\{3,4\},w_2), \ldots,\axp_1^1(\{\lfloor\frac{n}{2}\rfloor-1,
  \lfloor \frac{n}{2} \rfloor\},w_2) \}$.
  Using the same reasoning as above, we can see that: the number of AXps
  for $w_2$ is $m={\lfloor \frac{n}{2}
  \rfloor}\times \frac{1}{2}$ and all AXps are disjoint sets,
  for this case the number of CXps ($2^m$) is exponentially larger than the number
  of AXps ($m$) and the size of $\mathcal{A}$ (as an NFA) $O(m^2)$.
\end{proof}

\section{ExtractCXp}

The extraction of a minimal CXp is outlined in
Algorithm~\ref{alg:extract_cxp}.
The input is a candidate explanation $\mathcal{Y}$ where $\cxp(\YYY)$
is a weak CXp (includes at least one counterexample string to
$L(\mathcal{A})$),
an FA $\mathcal{A}$, and an accepted word $w$.
%
%
\textsc{ExtractCxp} builds a candidate CXp automaton $W$ from $\YYY$
by trying to replace $\any$ characters in $W$ by the corresponding
character in $w$.
It then checks whether $L(W) \subseteq L(\mathcal{A})$ which will only hold if
the $L(W)$ includes no counterexamples.
If $\YYY = \{1, \ldots, |w|\}$ this means that $\Sigma^n \subseteq
L(\mathcal{A})$ and there are \emph{no CXps} for $w \in L(\mathcal{A})$.
In this case, the function returns $\bot$ indicating there is no CXp.
%
%
Otherwise, it tries replacing the $\any$ character at each position
$i$ with $w[i]$ and checks if $L(W) \subseteq L(\mathcal{A})$.
If this is the case, we must leave position $i$ free, and we record
$i$ in $\mathcal{Y}_{min}$. At the end we have constructed a minimal
CXp since replacing any $\any$ character left over will leave no
counterexample.

\begin{algorithm}[tb]
  \caption{\textsc{ExtractCXp} -- a Single CXp Extraction}
\label{alg:extract_cxp}
\textbf{Input}: Candidate set $\YYY$, automaton $\mathcal{A}$, word $w$, bounds $l, u$ \\
\textbf{Output}: Minimal $\YYY$ with $\cxp_l^u(\YYY, w) \not\subseteq L(\mathcal{A})$

\begin{algorithmic}[1]

  \IF{$\cxp_l^u(\YYY, w) \subseteq L(\mathcal{A})$}
      \STATE \textbf{return} $\bot$
  \ENDIF
  \FORALL{$i \in \YYY$} 
    \IF{$\cxp_l^u(\YYY \setminus \{i\}, w) \not\subseteq L(\mathcal{A})$}
      \STATE $\YYY \gets \YYY \setminus \{i\}$
    \ENDIF
  \ENDFOR
  \STATE \textbf{return} $\YYY$
\end{algorithmic}
\end{algorithm}

Similarly to \textsc{ExtractAXp}, procedure \textsc{ExtractCXp} has
the complexity of $O(|\Sigma|mn^2)$.

\newtheorem{corollary}{Corollary}

\begin{corollary}
    \label{thm:cxp_appendix}
    Given a DFA $\mathcal{A}$ with $m$ states and a word $w$ s.t. $\lvert
    w\rvert=~n$, finding a CXp for $w \in L(\mathcal{A})$ using $\lwa_1^1$,
    $\lwa_1^\infty$ or $\lwa_0^\infty$ can be done in
    $O(\lvert\Sigma\rvert mn^2)$ time.
\end{corollary}
\begin{proof}
(Sketch) We can extract a CXp for $w \in L(\mathcal{A})$ using $\lwa_1^1$
using the call \textsc{ExtractCXp}$(1..|w|,\mathcal{A},w,l,u)$.
The correctness follows by construction, removing a position from $\mathcal{Y}$
only occurs when this maintians a CXp.
The complexity argument is essentially the same.
\end{proof}

\section{Additional Experiments}

As in the main paper, we compare the performance of different
approaches for computing all AXps: \textsc{XpEnum}($true$) uses the
\textsc{XpEnum} algorithm to find all AXps targeting AXps;
\textsc{XpEnum($true$)-Warm} as above but warm starting
with all singleton CXps; and \textsc{XpEnum}($false$) uses the
\textsc{XpEnum} algorithm to find all AXps targeting CXps.

\mysubsection{Maze using the languages $\lwa_0^\infty$ and $\lwa_1^\infty$.}
Similar to the experiments with mazes in the main paper,
here we consider the same mazes but we use the languages
$\lwa_0^\infty$ and $\lwa_1^\infty$. Note that explanations
are different for different languages. Figure \ref{fig:appendix_maze_cxps}
shows an example of a maze with a rejected path and its minimal cardinality
CXps using $\lwa_1^1$ and $\lwa_1^\infty$.
Using language $\lwa_1^1$ the CXp with minimal cardinality is $\{8, 11, 13\}$
while using $\lwa_1^\infty$ the CXp with minimal cardinality is $\{13\}$
because one replacement can introduce many symbols.

\begin{figure}[t]
  \begin{subfigure}[b]{0.325\columnwidth}
    \centering
    \includegraphics[height=4.5cm]{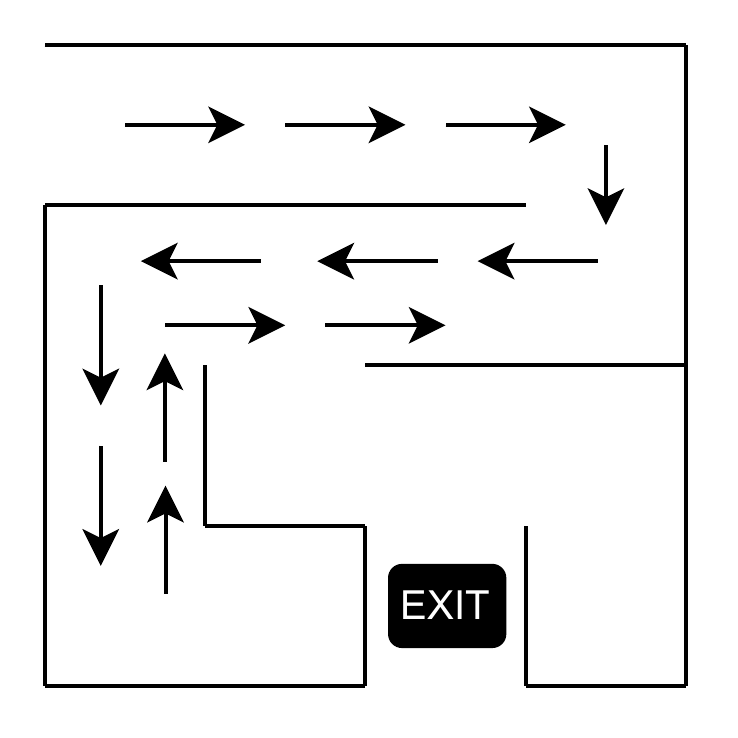}
    \caption{Example maze with a rejected path.}
    \label{fig:appendix_maze_original}
  \end{subfigure}
  \hfill
  \begin{subfigure}[b]{0.325\columnwidth}
    \centering
    \includegraphics[height=4.5cm]{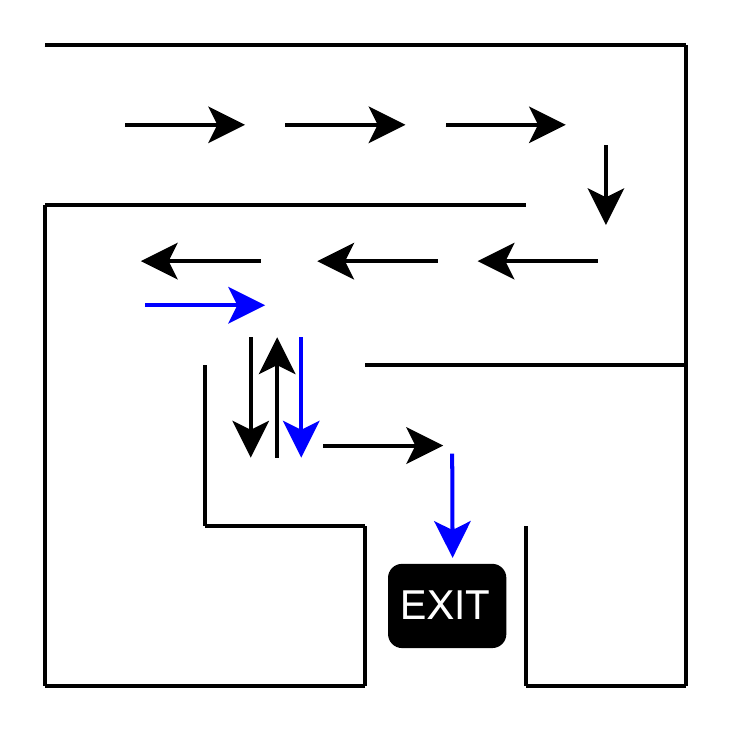}
    \caption{Minimal cardinality CXp using $\lwa_1^1$.}
    \label{fig:appendix_maze_11}
  \end{subfigure}
  \hfill
  \begin{subfigure}[b]{0.325\columnwidth}
    \centering
    \includegraphics[height=4.5cm]{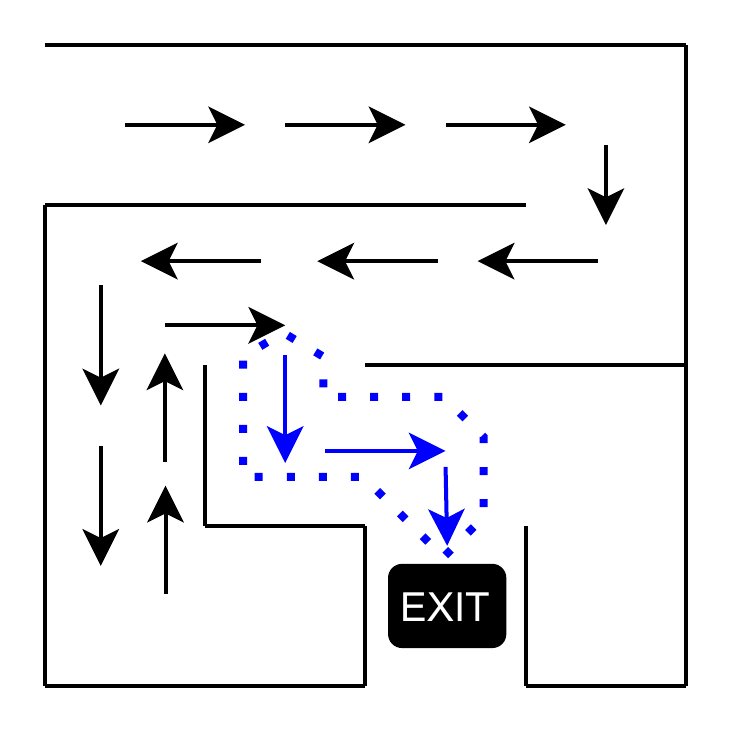}
    \caption{Minimal cardinality CXp using $\lwa_1^\infty$.}
    \label{fig:appendix_maze_infty}
  \end{subfigure}
  \caption{Maze with a rejected path and its minimal cardinality CXps using $\lwa_1^1$ and $\lwa_1^\infty$.}
  \label{fig:appendix_maze_cxps}
\end{figure}

As experiments in the main paper, to evaluate this case,
we generated mazes of different sizes, starting with
grids of size $10\times 10$ and increasing the dimension by 1 up to
$50\times 50$, getting $861$ different maze sizes.
We built mazes where each wall has 1 in
3 chance of being present. After the construction of each maze, we get a
possible solution path and introduce $\{1, 2, \dots, 5\}$ random changes in
the path, getting five different rejected paths for each maze. If the random
maze did not have any solution, we repeated the process until we found a
solvable maze for that dimension.

\begin{figure}[t]
  \begin{subfigure}[b]{0.49\columnwidth}
  \centering
  \includegraphics[width=\columnwidth]{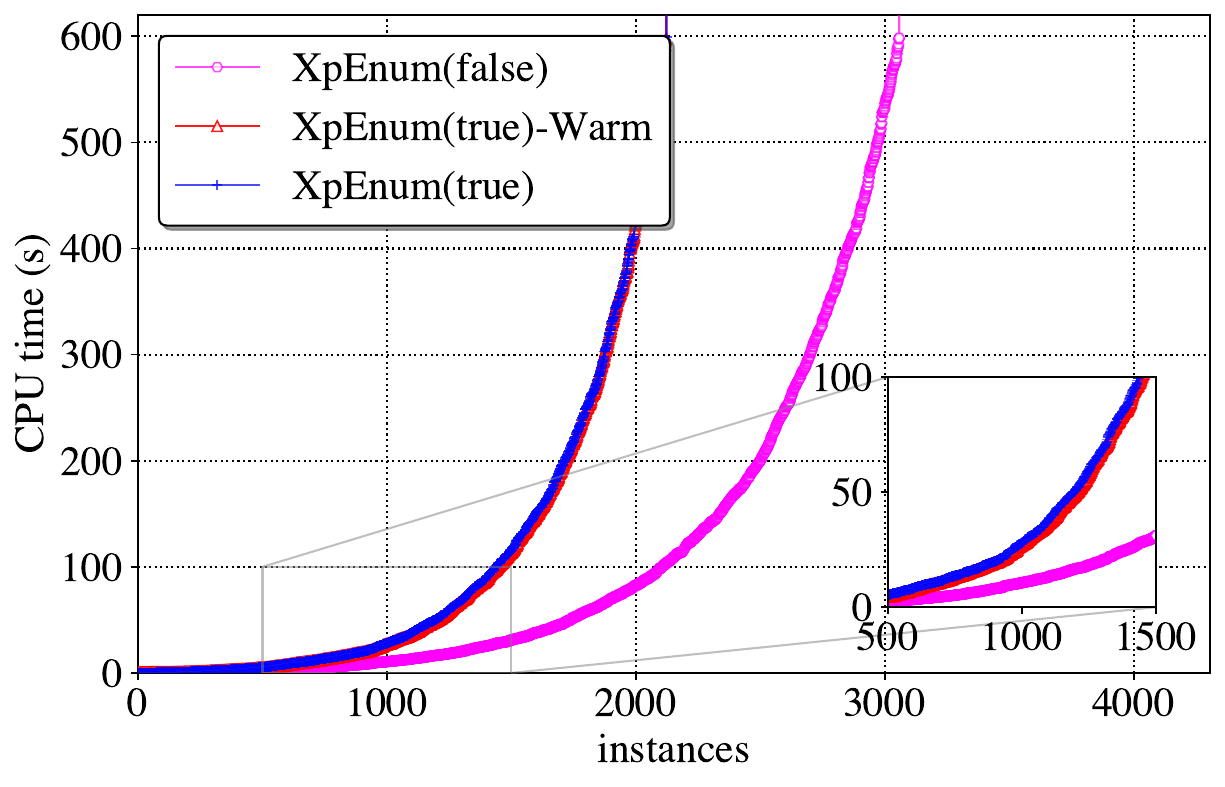}
    \caption{Enumerating all AXps using $\lwa_1^\infty$ for maze.}
    \label{fig:plot_maze_plus_all_exp}
  \end{subfigure}
  \hfill
  \begin{subfigure}[b]{0.49\columnwidth}
    \centering
    \includegraphics[width=\columnwidth]{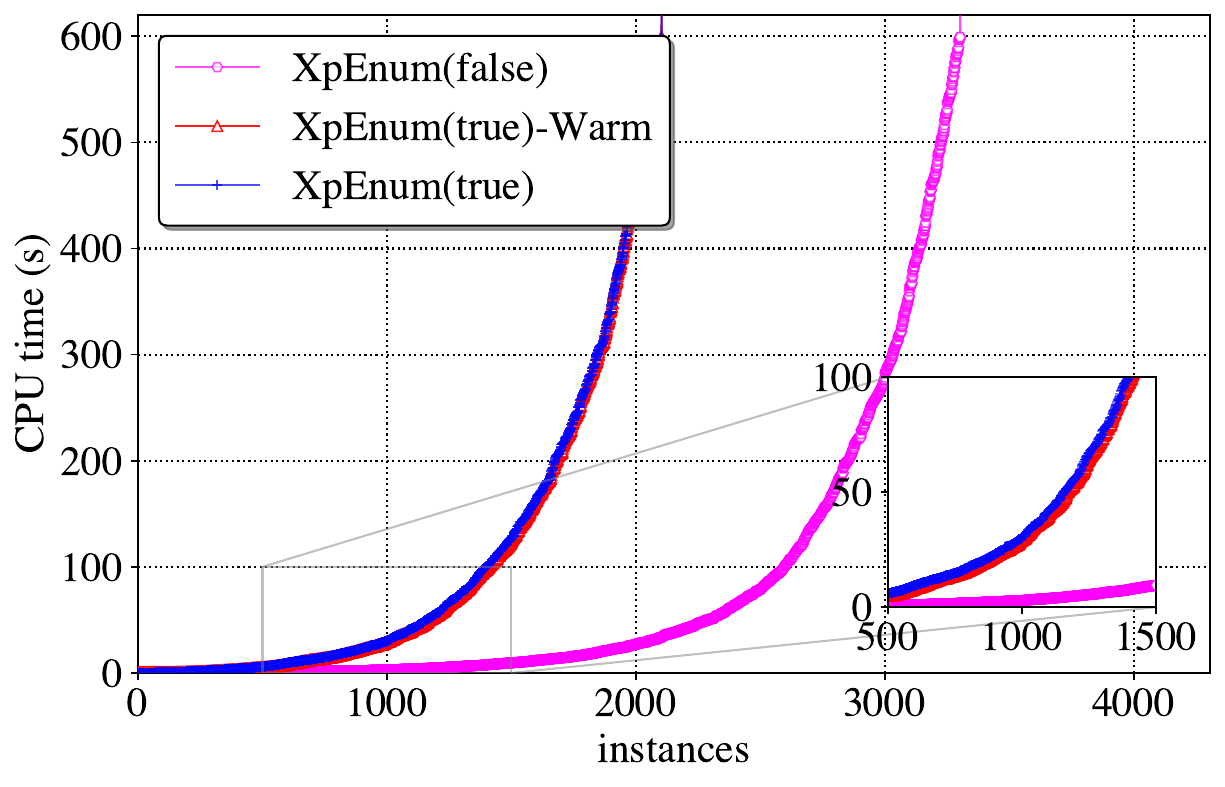}
    \caption{Enumerating all AXps using $\lwa_0^\infty$ for maze.}
    \label{fig:plot_maze_star_all_exp}
  \end{subfigure}
  \caption{Maze with a rejected path and its minimal cardinality CXps using $\lwa_1^1$ and $\lwa_1^\infty$.}
  \label{fig:appendix_maze_all_exp_infty}
\end{figure}

In the enumeration of all AXps shown in Figure~\ref{fig:appendix_maze_all_exp_infty}
we can see that the enumeration of all AXps using
$\lwa_1^\infty$ and $\lwa_0^\infty$ is more challenging than using $\lwa_1^1$
depicted in the main paper not
because because extracting a single AXp is harder, but because there
are more explanations to enumerate. Clearly for these examples targeting
CXps is the best approach.

\begin{table}[ht]
  \caption{Average number of explanations computed where enumeration
  process was completed using languages $\lwa_1^\infty$ and
  $\lwa_0^\infty$.}
  \label{tab:maze_plus_star_finished}
  \centering
  \begin{tabular}{ccccc}
  \toprule
  \multicolumn{1}{c}{}  & \multicolumn{2}{c}{\makecell[c]{$\lwa_1^\infty$}}  & \multicolumn{2}{c}{\makecell[c]{$\lwa_0^\infty$}}   \\

  \cmidrule{2-5}
  \multicolumn{1}{c}{}   &  AXp & CXp & AXp & CXp\\
  \midrule
  \textsc{XpEnum($false$)} &  6.3 & 3339.6 & 6.6 & 4328.5\\
  \midrule
  \textsc{XpEnum($true$)} &  5.0 & 585.8 & 5.0 & 573.4\\
  \midrule
  \textsc{XpEnum($true$)-Warm} & 5.0 &  581.3 & 5.0 & 570.1\\
  \bottomrule

\end{tabular}

\end{table}

Table~\ref{tab:maze_plus_star_finished} shows the average number of
explanations computed for instances where the enumeration process
was completed using languages $\lwa_1^\infty$ and $\lwa_0^\infty$.
In most cases, the number of CXps is higher than the number
of AXps. This is expected because using a languages $\lwa_1^\infty$
or $\lwa_0^\infty$ allows for greater flexibility in replacements,
resulting in more possible ways to fix an input word.
This data reveals the reasons why targeting CXps is the best
approach for the enumeration of all explanations for these
scenarios.
As studied in the main paper, usually targeting CXps is the preferable
when the problems get harder.
Tipically, matching a word against a pattern is
restricted leading to a limited number of AXps, but a large number of CXps.

\mysubsection{Generated corpus.} 

\begin{wrapfigure}[18]{R}{0.55\textwidth}
\centering
    \centering
    \includegraphics[width=0.55\columnwidth]{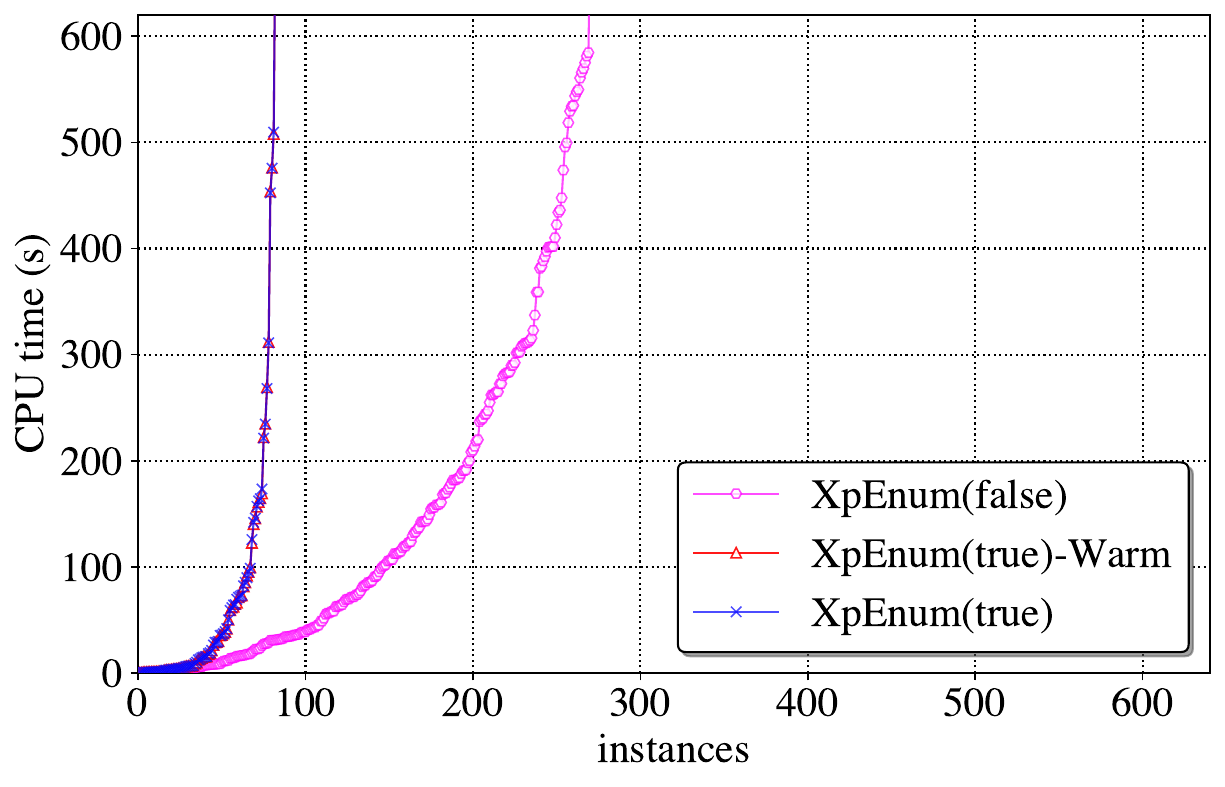}
    \caption{Generated corpus with  \textbf{rejected} words}
    \label{fig:rejected_random}
\end{wrapfigure}

The generated corpus of FAs
is the same as the one presented in the main paper,
and in this section, we analyse the behaviour for rejected words using
language $\lwa_1^1$.
For the generated corpus, we generated $m$ random words $M$ with length $l$ over the
alphabet $\{1,\ldots,d\}$ for $l\in \{5,10,15,20\}$, $m\in
\{1,3,5,10\}$, $d\in \{2,3,5,10\}$.
Figure~\ref{fig:rejected_random} shows the results for the enumeration of all
AXps for the rejected words.
Then, we build an FA for each configuration such that it accepts the
language $\{\Sigma^* w \Sigma^* \mid w \in M\}$.

Many patterns in this corpus are easy to match with long words.
For example, with $l = 5$, $m = 1$, $d = 2$ a possible pattern
is \texttt{(a|b)$^*$abbba(a|b)$^*$}.
A long word that does not match this pattern may have a
large number of didsjoint CXps, as many positions can be
modified to match the pattern.
Impliying that the number of AXps is exponential in the number of CXps
(Proposition~\ref{thm:exponential_appendix}).

Table~\ref{tab:table_dont_finished} shows, 
for the cases that did not finish within the time limit,
the reason for the timeout is the large number of 
CXPs, and even more AXps, often growing exponentially.

\begin{table}[tb]
  \caption{Average number of explanations computed if reaching the timeout.}
  \label{tab:table_dont_finished}
  \centering
  \begin{tabular}{ccc}
    \hline
    \multicolumn{1}{c}{}  & \multicolumn{2}{c}{\makecell[c]{Corpus (rejected)}}   \\

    \cmidrule{2-3}
    \multicolumn{1}{c}{}   &  AXp & CXp \\
    \midrule
    \textsc{XpEnum($false$)} &  254.0 & 1.0\\
    \midrule
    \textsc{XpEnum($true$)} &  51976.7 & 540.2 \\
    \midrule
    \textsc{XpEnum($true$)-Warm} &  52964.1 & 543.0\\
    \bottomrule

  \end{tabular}

\end{table}

\end{document}